\documentclass[twoside]{IEEEtran}
\usepackage[T1]{fontenc}

\ifCLASSINFOpdf
\else
\fi
\usepackage{amsmath}
\usepackage[cmintegrals]{newtxmath}
\usepackage[utf8]{inputenc} 
\usepackage{url}            
\usepackage{booktabs}       
\usepackage{amsfonts}       
\usepackage{nicefrac}       
\usepackage{microtype}      

\usepackage{graphicx}
\usepackage{amsmath}
\usepackage{tcolorbox}
\usepackage{multirow}
\usepackage{layouts}
\usepackage[caption=false,font=footnotesize]{subfig} 

\usepackage{algorithm}
\usepackage[noend]{algorithmic} 

\usepackage{mathtools}

\DeclarePairedDelimiter{\floor}{\lfloor}{\rfloor}

\newcommand{\name}{Turbo-Aggregate} 
\newcommand{\Nl}{N_l} 

\usepackage{amsthm}

\usepackage[noadjust]{cite}

\newtheorem{theorem}{Theorem}

\newtheorem{remark}{Remark}

\colorlet{Mycolor1}{black}
\colorlet{Mycolor2}{black}
\colorlet{Mycolor3}{black}

\newcommand{\RNum}[1]{\uppercase\expandafter{\romannumeral #1\relax}}

\begin{document}
%
\title{{\name}: Breaking the Quadratic Aggregation Barrier in Secure Federated Learning}
%
%
%


\author{Jinhyun So,\quad Ba\c{s}ak~G{\"u}ler,~\IEEEmembership{Member,~IEEE} \quad and \quad A. Salman Avestimehr,~\IEEEmembership{Fellow,~IEEE}
\thanks{
    %
    Jinhyun So is with the Department of Electrical and Computer Engineering, University of Southern California, Los Angeles, CA, 90089 USA (e-mail: jinhyuns@usc.edu). Ba\c{s}ak G{\"u}ler is with the Department of Electrical and Computer Engineering, University of California, Riverside, CA, 92521 USA (email: bguler@ece.ucr.edu). 
    A. Salman Avestimehr is with the Department of Electrical and Computer Engineering, University of Southern California, Los Angeles, CA, 90089 USA (e-mail: avestimehr@ee.usc.edu).
    
    This work is published in IEEE Journal on Selected Areas in Information Theory \cite{so2021turbo}.
    }
}

\maketitle


\vspace{-1.2cm}
\begin{abstract} 
Federated learning is a distributed framework for training machine learning models over the data residing at mobile devices, while protecting the privacy of individual users. A major bottleneck in scaling federated learning to a large number of users is the overhead of secure model aggregation across many users. In particular, the overhead of the state-of-the-art protocols for secure model aggregation grows quadratically with the number of users. In this paper, we propose the first secure aggregation framework, named {\name}, that in a network with $N$ users achieves a secure aggregation overhead of $O(N\log{N})$, as opposed to $O(N^2)$, while tolerating up to a user dropout rate of $50\%$. {\name} employs a multi-group circular strategy for efficient model aggregation, and leverages additive secret sharing and novel coding techniques for injecting aggregation redundancy in order to handle user dropouts while guaranteeing user privacy. We experimentally demonstrate that {\name} achieves a total running time that grows almost linear in the number of users, and provides up to $40\times$ speedup over the state-of-the-art protocols with up to $N=200$ users. Our experiments also demonstrate the impact of bandwidth on the performance of {\name}.
\end{abstract}

\vspace{-0.1cm}

%
\IEEEpeerreviewmaketitle

\section{Introduction}
Federated learning is an emerging approach that enables model training over a large volume of decentralized data residing in mobile devices, while protecting the privacy of the individual users \cite{mcmahan2016communication, bonawitz2016practical, bonawitz2017practical, kairouz2019advances}. This is achieved by two key design principles. First, the training data is kept on the user device rather than sending it to a central server, and users locally perform model updates using their individual data. Second, local models are aggregated in a privacy-preserving framework, either at a central server (or in a distributed manner across the users) to update the global model. 
The global model is then pushed back to the mobile devices for inference. 
This process is demonstrated in Figure~\ref{federated}. 

\begin{figure}[t]
\centering
\includegraphics[width=0.8\linewidth]{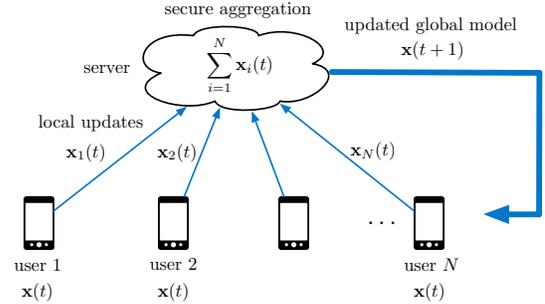}
\vspace{-0.1cm}
\caption{Federated learning framework. At iteration $t$, the central server sends the current version of the global model, $\mathbf{x}(t)$, to the mobile users. User $i\in[N]$ updates the global model using its local data, and computes a local model $\mathbf{x}_i(t)$. The local models are then  aggregated in a privacy-preserving manner. Using the aggregated models, the central server updates the global model $\mathbf{x}(t+1)$ for the next round, and pushes it back to the mobile users.}
\label{federated}
\vspace{-0.3cm}
\end{figure}

The privacy of individual models in federated learning is protected through what is known as a \emph{secure aggregation} protocol \cite{bonawitz2016practical, bonawitz2017practical}. 
In this protocol, each user locally masks its own model using pairwise random masks and sends the masked model to the server. The pairwise masks have a unique property that once the masked models from all users are summed up at the server, the pairwise masks cancel out. 
As a result, the server learns the aggregate of all models, but no individual model is revealed to the server during the process. 
This is a key property for ensuring user privacy in secure federated learning. 
In contrast, conventional distributed training setups that do not employ secure aggregation may reveal extensive information about the private datasets of the users, which has been recently shown in \cite{NIPS2019_9617, 8737416,geiping2020inverting}.
To prevent such information leakage, secure aggregation protocols ensure that the individual update of each user is kept private, both from other users and the central server \cite{bonawitz2016practical, bonawitz2017practical}.
A recent promising implementation of federated learning, as well as its application to Google keyboard query suggestions is demonstrated in \cite{yang2018applied}. Several other works have also demonstrated that leveraging the information that is distributed over many mobile users can increase the training performance dramatically, while ensuring data privacy and locality \cite{mcmahan2017learning,bonawitz2019towards, li2020federated}.

The overhead of secure model aggregation, however, creates a  major bottleneck in scaling secure federated learning to a large number of users.
More specifically, in a network with $N$ users, the state-of-the-art protocols for secure aggregation require pairwise random masks to be generated between each pair of users (for hiding the local model updates), and therefore the overhead of secure aggregation grows quadratically in the number of users (i.e., $O(N^2)$) 
\cite{bonawitz2016practical, bonawitz2017practical}. 
This quadratic growth of secure aggregation overhead limits its practical applications to hundreds of users while the scale of current mobile systems is in the order of tens of millions \cite{bonawitz2019towards}.

Another key challenge in model aggregation is the dropout or unavailability of the users. 
Device availability and connection quality in mobile networks change rapidly, and users may drop from federated learning systems at any time due to various reasons, such as poor connectivity, making a phone call, low battery, etc. 
The design protocol hence needs to be robust to operate in such environments, where users can drop at any stage of the protocol execution. Furthermore, dropped or delayed users can lead to privacy breaches \cite{bonawitz2017practical}, and privacy guarantees should hold even in the case when users are dropped or delayed.



In this paper, we introduce a novel secure aggregation framework for federated learning, named  {\name}, with four salient features:  
\begin{enumerate}
\item {\name} reduces the overhead of secure aggregation to $O(N\log{N})$ from  $O(N^2)$; 

\item {\name} has provable robustness guarantees against up to a user dropout rate of $50\%$; 

\item {\name} protects the privacy of the local model updates of each individual user, in the strong information-theoretic sense; 

\item {\name} experimentally achieves a total running time that grows almost linear in the number of users, and provides up to $40\times$ speedup over the state-of-the-art with $N=200$ users, in distributed implementation over Amazon EC2 cloud.    
\end{enumerate}

At a high level, {\name} is composed of three main components. First, {\name} employs a multi-group circular strategy for model aggregation. 
In particular, the users are partitioned into several groups, and at each aggregation stage, the users in one group pass the aggregated models of all the users in the previous groups and current group to users in the next group. We show that this structure enables the reduction of aggregation overhead to $O(N\log{N})$ (from $O(N^2)$). 
However, there are two key challenges that need to be addressed in the proposed multi-group circular strategy for model aggregation. The first one is to protect the privacy of the individual user, i.e., the aggregation protocol should not allow the identification of individual model updates. The second one is handling the user dropouts. 
For instance, a user dropped at a higher group of the protocol may lead to the loss of the aggregated model information from all the previous groups, and collecting this information again from the lower groups may incur a large communication overhead. 

The second key component is to leverage additive secret sharing~\cite{beimel2011secret, evans2018pragmatic} to enable privacy and security of the users. 
In particular, additive sharing masks each local model by adding randomness in a way that can be cancelled out once the models are aggregated. 
Finally, the third component is to add aggregation redundancy via Lagrange coding~\cite{yu2018lagrange} to enable robustness against delayed or dropped users. In particular, {\name} injects redundancy via Lagrange polynomial so that the added redundancy can be exploited to reconstruct the aggregated model amidst potential dropouts.

{\name} allows the use of both centralized and decentralized communication architectures. 
The centralized architecture refers to the communication model used in the conventional federated learning setup where all communication goes through a central server, i.e., the server acts as an access point \cite{mcmahan2016communication, bonawitz2017practical, kairouz2019advances}. The decentralized architecture, on the other hand, refers to the setup where mobile devices communicate directly with each other via an underlay communication network (e.g., a peer-to-peer network) \cite{lalitha2019peer, he2019central} without requiring a central server for secure model aggregation. {\name} also allows additional parallelization opportunities for communication, such as broadcasting and multi-casting.

We theoretically analyze the performance guarantees of {\name} in terms of the aggregation overhead, privacy protection, and robustness to dropped or delayed users. 
In particular, we show that {\name} achieves an aggregation overhead of $O(N\log{N})$ and can tolerate a user dropout rate of $50\%$. 
We then quantify the privacy guarantees of our system. 
An important implication of dropped or delayed users is that they may lead to privacy breaches \cite{bonawitz2016practical}. 
Accordingly, we show that the privacy-protection of our algorithm is preserved in such scenarios, i.e., when users are dropped or delayed. 

We also provide extensive experiments to numerically evaluate the performance of {\name}. 
To do so, we implement {\name} for up to $200$ users on the Amazon EC2 cloud, and compare its performance with the state-of-the-art secure aggregation protocol from \cite{bonawitz2017practical}.
We demonstrate that {\name} can achieve an overall execution time that grows almost linear in the number of users, and provides up to $40\times$ speedup over the state-of-the-art with $200$ users. Furthermore, the overall execution time of {\name} remains stable as the user dropout rate increases, while for the benchmark protocol, the overall execution time significantly increases as the user dropout rate increases. 
We further study the impact of communication bandwidth on the performance of {\name}, by measuring the total running time with various bandwidth constraints. 
Our experimental results demonstrate that {\name} still provides substantial gain in environments with more severe bandwidth constraints.

\section{Related Work}
A potential solution for secure aggregation is to leverage cryptographic approaches, such as multiparty computation (MPC), homomorphic encryption, or differential privacy.
MPC-based techniques mainly utilize Yao's garbled circuits or secret sharing (e.g., \cite{yao1982protocols, shamir1979share, ben1988completeness, beerliova2008perfectly}). 
Their main bottleneck is the high communication cost, and communication-efficient implementations require an extensive offline computation part \cite{ben1988completeness, beerliova2008perfectly}. 
A notable recent work is \cite{burkhart2010sepia}, which focuses on optimizing MPC protocols for network security and monitoring. 
Homomorphic encryption is a cryptographic secure computation scheme that allows aggregations to be performed on encrypted data \cite{leontiadis2014private, rastogi2010differentially, halevi2011secure}. However, the  privacy guarantees of homomorphic encryption depends on the size of the encrypted data (more privacy requires a larger encypted data size), and performing computations in the encrypted domain is computationally expensive \cite{gentry2009fully, damgaard2012multiparty}. 
Differential privacy is a a noisy release mechanism that preserves the privacy of personally identifiable information, in that the removal of any single element from the dataset does not affect the computation outcomes significantly. As such, the computation outcomes cannot be used to infer much about any single individual element \cite{dwork2006calibrating}.
In the context of federated learning, differential privacy is mainly used to ensure that individual data points from the local datasets cannot be identified from the local updates sent to the server, by adding artificial noise to the local updates at the clients’ side \cite{geyer2017differentially,mcmahan2017learning,wei2020federated}.
This approach entails a trade-off between convergence performance and privacy protection, i.e., stronger privacy guarantees lead to a degradation in the convergence performance.  
On the other hand, our focus is on ensuring that the server or a group of colluding users can learn nothing beyond the aggregate of all local updates, while preserving the accuracy of the  model. This approach, also known as secure aggregation \cite{bonawitz2016practical, bonawitz2017practical}, does not sacrifice the convergence performance.

A recent line of work has focused on secure aggregation by additive masking  \cite{acs2011have}, \cite{bonawitz2017practical}. In \cite{acs2011have}, users agree on pairwise secret keys using a Diffie-Hellman type key exchange protocol and then each user sends the server a masked version of their data, which contains the pairwise masks as well as an individual mask. 
The server can then sum up the masked data received from the users to obtain the aggregated value, as the summation of additive masks cancel out. 
If a user fails and drops out, the server asks the remaining users to send the sum of their pairwise keys with the dropped users added to their individual masks, and subtracts them from the aggregated value. The main limitation of this protocol is the communication overhead of this recovery phase, as it requires the entire sum of the missing masks to be sent to the server. Moreover, the protocol terminates if additional users drop during this phase. 

A novel technique is proposed in \cite{bonawitz2017practical} to ensure that the protocol is robust if additional users drop during the recovery phase. It also ensures that the additional information sent to the server does not breach privacy. 
To do so, the protocol utilizes pairwise random masks between users to hide the individual models. The cost of reconstructing these masks, which takes the majority of execution time, scales with respect to $O(N^2)$, with $N$ corresponding to the number of users. 
The execution time of \cite{bonawitz2017practical} increases as more users are dropped, as the protocol requires additional information corresponding to the dropped users. 
The recovery phase of our protocol does not require any additional information to be shared between the users, which is achieved by a coding technique applied to the additively secret shared data. Hence, the execution time of our algorithm stays almost the same as more and more users are dropped, the only overhead comes from the decoding phase whose contribution is very small compared to the overall communication cost.

Notable approaches to reduce the communication cost in federated learning include reducing the model size via quantization, or learning in a smaller parameter space \cite{konevcny2016federated}. In \cite{bonawitz2019federated}, a framework has been proposed for autotuning the parameters in secure federated learning, to achieve communication-efficiency. Another line of work has focused on approaches based on decentralized learning \cite{he2018cola, lian2017can} or edge-assisted hierarchical physical layer topologies \cite{liu2019edge}. Specifically, \cite{liu2019edge} utilizes edge servers to act as an intermediate aggregator for the local updates from edge devices. The global model is then computed at the central server by aggregating the intermediate computations available at the edge servers. 
These setups perform the aggregation using the clear (unmasked) model updates, i.e., the aggregation is not required to preserve the privacy of individual model updates.  
Our focus is different, as we study the secure aggregation problem which requires the server to learn no information about an individual update beyond the aggregated values. Finally, approaches that aim at alleviating the aggregation overhead by reducing the model size (e.g., quantization \cite{konevcny2016federated}) can also be leveraged in {\name}, which can be an interesting future direction.

Circular communication and training architectures have been considered previously in the context of distributed stochastic gradient descent on clear (unmasked) gradient updates, to reduce communication load \cite{li2018pipe} or to model data-heterogeneity \cite{eichner2019semi}. Different from these setups, our key challenge in this work is handling user dropouts while ensuring user privacy, i.e., secure aggregation. 
Conventional federated learning frameworks consider a centralized communication architecture in which all communication between the mobile devices goes through a central server \cite{bonawitz2017practical, mcmahan2016communication, kairouz2019advances}. More recently, decentralized federated learning architectures without a central server have been considered for peer-to-peer learning on graph topologies \cite{lalitha2019peer} and in the context of social networks \cite{he2019central}. 
Model poisoning attacks on federated learning architectures have been analyzed in \cite{bhagoji2018analyzing, fang2019local}. Differentially-private federated learning frameworks have been studied in \cite{geyer2017differentially, sun2019can}. A multi-task learning framework for federated learning has been proposed in \cite{smith2017federated}, for learning several models simultaneously. \cite{mohri2019agnostic, li2019fair} have explored federated learning frameworks to address fairness challenges and to avoid biasing the trained model towards certain users.  Convergence properties of trained models are studied in \cite{li2019convergence}.

\section{System Model}~\label{2-ProbSetting}

In this section, we first discuss the basic federated learning model. Next, we introduce the secure aggregation protocol for federated learning and discuss the key parameters for performance evaluation.
Finally, we present the state-of-the-art for secure aggregation. 

\subsection{Basic Federated Learning Model}\label{basicFL}

Federated learning is a distributed learning framework that allows training machine learning models directly on the data held at distributed devices, such as mobile phones. 
The goal is to learn a single global model $\mathbf{x}$ with dimension $d$, using data that is generated, stored, and processed locally at millions of remote devices. This can be represented by minimizing a global objective function,
\begin{equation}\label{eq:objective_fnc}
\min_{\mathbf{x}} L(\mathbf{x}) \text{ such that }  L(\mathbf{x}) = \sum_{i=1}^N w_i L_i (\mathbf{x}), 
\end{equation} 
where $N$ is the total number of mobile users, $L_i$ is the local objective function of user $i$, and $w_i \geq 0$ is a weight parameter assigned to user $i$ to specify the relative impact of each user such that $\sum_{i}w_i = 1$. One natural setting of the weight parameter is $w_i=\frac{m_i}{m}$ where $m_i$ is the number of samples of user $i$ and $m=\sum_{i=1}^N m_i$\footnote{For simplicity, we assume that all users have equal-sized datasets i.e., a weight parameter assigned to user $i$ satisfies $w_i=\frac{1}{N}$ for all $i\in[N]$.}.


To solve \eqref{eq:objective_fnc}, conventional federated learning architectures consider a centralized communication topology in which all communication between the individual devices goes through a central server \cite{mcmahan2016communication, bonawitz2017practical, kairouz2019advances}, and no direct links are allowed between the mobile users.
The learning setup is as demonstrated in Figure~\ref{federated}.
At iteration $t$, the central server shares the current version of the global model, $\mathbf{x}(t)$, with the mobile users. 
Each user then updates the model using its local data. 
User $i\in[N]$ then computes a local model $\mathbf{x}_i(t)$.
To increase communication efficiency, each user can update the local model over multiple local epochs before sending it to the server \cite{mcmahan2016communication}.
The local models of the $N$ users are sent to the server and then aggregated by the server. 
Using the aggregated models, the server updates the global model $\mathbf{x}(t+1)$ for the next iteration.
This update equation is given by
\begin{equation}
    \mathbf{x}(t+1) = \sum_{i\in \mathcal{U}(t)} \mathbf{x}_i(t), \label{eq:update_eq}
\end{equation}
where $\mathcal{U}(t)$ denotes the set of participating users at iteration $t$.
Then, the server pushes the updated global model $\mathbf{x}(t+1)$ to the mobile users.

\subsection{Secure Aggregation Protocol for Federated Learning and Key Parameters}

The basic federated learning model from Section~\ref{basicFL} aims at addressing the 
privacy concerns over transmitting raw data to the server, by letting the training data remain on the user device and instead requiring only the local models to be sent to the server. 
However, as the local models still carry extensive information about the local datasets stored at the users, the server can reconstruct the private data from the local models by using a model inversion attack, which has been  recently demonstrated in \cite{NIPS2019_9617, 8737416,geiping2020inverting}. Secure aggregation has been introduced  in \cite{bonawitz2017practical} to address such privacy leakage from the local models. A secure aggregation protocol enables  the computation of the aggregation operation in  \eqref{eq:update_eq} while ensuring that the server learns no information about the local models $\mathbf{x}_i(t)$ beyond their aggregated value $\sum_{i=1}^N \mathbf{x}_i(t)$.
In this paper, our focus is on the aggregation phase in \eqref{eq:update_eq} and how to make this aggregation phase secure and efficient.
In particular, our goal is to evaluate the aggregate of the local models
\begin{equation} \label{eq:goal_eq}
    \mathbf{z} = \sum_{i\in \mathcal{U}} \mathbf{x}_i, 
\end{equation}
where we omit the iteration index $t$ for simplicity. 
As we discuss in Section~\ref{state-of-the-art} and Appendix A in detail, secure aggregation protocols build on cryptographic primitives that require all operations to be carried out over a finite field. 
Accordingly, similar to prior works \cite{bonawitz2016practical, bonawitz2017practical}, we assume that the elements of $\mathbf{x}_{i}^{(l)}$ and $\mathbf{z}$ are from a finite field $\mathbb{F}_q$ for some field size $q$. 

We evaluate the performance of a secure aggregation protocol for federated learning through the following key parameters.

\begin{enumerate}
    \item \textit{Robustness guarantee:}
    We consider a network model in which each user can drop from the network with a probability $p\in[0,1]$, called the user dropout rate. In a real world setting, the dropout rate varies between $0.06$ and $0.1$~\cite{bonawitz2019towards}. 
    The robustness guarantee quantifies the maximum user dropout rate that a protocol can tolerate with a probability approaching to $1$ as $N \rightarrow \infty$ to correctly evaluate the aggregate of the surviving user models.
    
    \item \textit{Privacy guarantee:}
    We consider a security model where the users and the server are honest but curious. We assume that up to $T$ users can collude with each other as well as with the server for learning the models of other users. The privacy guarantee quantifies the maximum number of colluding entities that the protocol can tolerate for the individual user models to keep private.
    
    \item \textit{Aggregation overhead:}
    The aggregation overhead, denoted by $C$, quantifies the asymptotic time complexity (i.e., runtime) with respect to the number of mobile users, $N$, for aggregating the models of all users in the network. Note that this includes both the computation and communication time complexities.     
\end{enumerate}

\subsection{State-of-the-art for Secure Aggregation}\label{state-of-the-art}

The state-of-the-art for secure aggregation in federated learning is the protocol proposed in  \cite{bonawitz2017practical}. 
In this protocol, each mobile user locally trains a model. By using pairwise random masking, the local models are securely aggregated through a central server, who then updates the global model. 
We present the details of the state-of-the-art in Appendix~A.
This protocol achieves robustness guarantee to user dropout rate of up to $p=0.5$, while providing privacy guarantee to up to $T=\frac{N}{2}$ colluding users.
However, its aggregation overhead is quadratic with the number of users (i.e., $C=O(N^2)$). 
This quadratic aggregation overhead severely limits the network size for real-world applications \cite{bonawitz2019towards}. 


Our goal in this paper is to develop a secure aggregation protocol that can provide comparable robustness and privacy guarantees as the state-of-the-art, while achieving a significantly lower (almost linear) aggregation overhead. 

\section{The {\name} Protocol}~\label{3-Algorithm}
We now introduce the {\name} protocol for secure federated learning that can simultaneously achieve robustness guarantee to a user dropout rate of up to $p=0.5$, privacy guarantee to up to $T=\frac{N}{2}$ colluding users, and aggregation overhead of $C=O(N\log{N})$.
{\name} is composed of three main components. 
First, it creates a multi-group circular aggregation structure for fast model  aggregation. Second, it leverages  additive secret sharing by adding randomness in a way that can be cancelled out once the models are aggregated, in order to guarantee the privacy of the users. Third, it adds aggregation redundancy via Lagrange polynomial in the model updates that are passed from one group to the next, so that the added redundancy can be exploited to reconstruct the aggregated model amidst potential user dropouts. 
We now describe each of these components in detail. 
An illustrative example is also presented in Appendix~B to demonstrate the execution of {\name}.

\vspace{-0.2cm}
\subsection{Multi-group circular aggregation}
\label{sec:multi-group} 
{\name} computes the aggregate of the  individual user models by utilizing a circular aggregation strategy. 
Given a mobile network with $N$ users, this is done by first partitioning the users into $L$ groups as shown in Figure~\ref{phones}, with $N_l$ users in group $l\in[L]$, such that $\sum_{l\in[L]}N_{l}=N$. 
We consider a random partitioning strategy in which each user is assigned to one of the available groups uniformly at random, 
by using a bias-resistant public randomness generation protocol such as in  \cite{syta2017scalable}. 
We use $\mathcal{U}_l\subseteq[N_l]$ to represent the set of users that complete their part in the protocol (surviving users), and $\mathcal{D}_l = [N_l]\backslash \mathcal{U}_l$ to denote the set of dropped users\footnote{
    For modeling the user dropouts, we focus on the worst-case scenario, which is the case 
    when a user drops during the execution of the corresponding group, i.e.,
    when a user receives messages from the previous group but fails to propagate it to the next group. 
}.
We use $\mathbf{x}_{i}^{(l)}$ to denote the local model of user $i$ in group $l\in[L]$, which is a vector of dimension $d$ that corresponds to the parameters of their locally trained model. 
Then, we can rewrite  \eqref{eq:goal_eq} as
\begin{equation}\label{eq:goal_eq2}
\mathbf{z} = \sum_{l\in[L]}\sum_{i\in\mathcal{U}_l} \mathbf{x}_{i}^{(l)}.
\end{equation} 
The elements of $\mathbf{x}_{i}^{(l)}$ and $\mathbf{z}$ are from a finite field $\mathbb{F}_q$ for some field size $q$. 
All operations are carried out over the finite field and we omit the modulo $q$ operation for simplicity.

\begin{figure}
\centering
\includegraphics[width=0.7\linewidth]{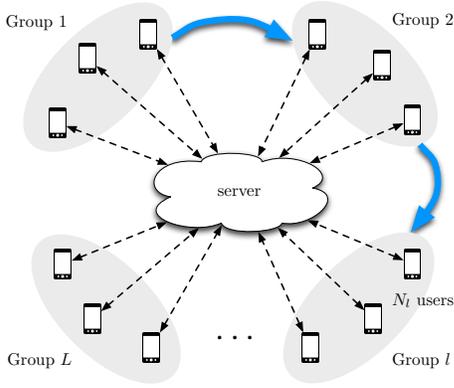}
\vspace{-0.3cm}\caption{Network topology with $N$ users partitioned to $L$ groups, with $N_l$ users in group $l\in[L]$. }
\label{phones}
\end{figure}
%
%
The dashed links in Figure~\ref{phones} represent the communication links between the server and mobile users. In our general description, we assume that all communication takes place through a central server, via creating pairwise secure keys using a Diffie-Hellman type key exchange protocol \cite{diffie1976new} as in \cite{bonawitz2017practical}. 
{\name} can also use decentralized communication architectures with direct links between devices, such as peer-to-peer communication, where users can communicate directly through an underlay communication network~\cite{lalitha2019peer, he2019central}. 
Then, the aggregation steps are the same as the centralized setting except that messages are now communicated via direct links between the users, and a random election algorithm should be carried out to select one user (or multiple users, depending on the application) to aggregate the final sum at the final stage instead of the server. The detailed process of the final stage will be explained in Section~\ref{sec:proposed_finalaggregation}.


{\name} consists of $L$ execution stages performed sequentially. 
At stage $l\in[L]$, users in group $l$ encode their inputs, including their trained models and the partial summation of the models from lower stages, and send them to users in group $l+1$. 
Next, users in group $l+1$ recover (decode) the missing information due to potentially dropped users, and then aggregate the received messages. 
At the end of the protocol, models of all surviving users will be aggregated. 



The proposed coding and aggregation mechanism guarantees that no party (mobile users or the server) can learn an individual model, or a partial aggregate of a subset of models. 
The server learns nothing but the final aggregated model of all surviving users. 
This is achieved by leveraging additive secret sharing to mask the individual models, which we  describe in the following.  

\subsection{Masking with additive secret sharing} 
{\name} hides the individual user models using additive masks to protect their privacy against potential collusions between the interacting parties. This is done by a two-step procedure. In the first step, the server sends a random mask to each user, denoted by a random vector $\mathbf{u}_i^{(l)}$ for user $i\in[N_l]$ at group $l\in[L]$. Each user then masks its local model $\mathbf{x}_{i}^{(l)}$ as  $\mathbf{x}_{i}^{(l)}+\mathbf{u}_i^{(l)}$. Since this random mask is known only by the server and the corresponding user, it protects the privacy of each user against potential collusions between any subset of the  remaining users, as long as the server is honest.
On the other hand, privacy may be breached if the server is adversarial and colludes with a subset of users. The second step of {\name} aims at protecting user privacy against such scenarios. 
In this second step, users generate additive secret sharing of the individual models for privacy protection against potential collusions between the server and the users. To do so, user $i$ in group $l$ sends a masked version of its local model to each user $j$ in group $l+1$, given by
\begin{equation}\label{eq:gen_SS}
    \widetilde{\mathbf{x}}_{i,j}^{(l)} = \mathbf{x}_{i}^{(l)}+\mathbf{u}_i^{(l)} + \mathbf{r}_{i,j}^{(l)},
\end{equation}
for $j\in[N_{l+1}]$, where  $\mathbf{r}_{i,j}^{(l)}$ is a random vector such that $\sum_{j\in[N_{l+1}]}\mathbf{r}_{i,j}^{(l)} = 0$ for all $i\in[N_l]$.
The role of additive secret sharing is not only to mask the model to provide privacy against collusions between the server and the users, but also to maintain the accuracy of aggregation by making the sum of the received data over the users in each group equal to the original data, 
as the vectors $\mathbf{r}^{(l)}_{i,j}$ cancel out. 

In addition, each user holds a variable corresponding to the aggregated masked models from the previous group. 
For user $i$ in group $l$, this variable is represented by $\widetilde{\mathbf{s}}_i^{(l)}$. 
At each stage of {\name}, users in the active group update and propagate these variables to the next group.  
Aggregation of these masked models is defined via the recursive relation, 
\begin{equation}\label{eq:sum_rx_data}
    \widetilde{\mathbf{s}}_{i}^{(l)} = \frac{1}{N_{l-1}}\sum_{j\in[N_{l-1}]}\widetilde{\mathbf{s}}_{j}^{(l-1)} + \sum_{j\in\mathcal{U}_{l-1}} \widetilde{\mathbf{x}}_{j,i}^{(l-1)}
\end{equation}
at user $i$ in group $l>1$, whereas the initial aggregation at group $l=1$ is set as  $\widetilde{\mathbf{s}}_{i}^{(1)}=\mathbf{0}$, for $i\in[N_1]$. 
While computing \eqref{eq:sum_rx_data}, any missing values in $\{\widetilde{\mathbf{s}}_{j}^{(l-1)}\}_{j\in[N_{l-1}]}$ (due to the users dropped in group $l-1$) is reconstructed via the recovery technique presented in Section~\ref{sec:coding}. 

User $i$ in group $l$ then sends the aggregated value in \eqref{eq:sum_rx_data} to each user in group $l+1$.
The average of the aggregated values from the users in group $l$ consists of the models of the users up to group $l-1$, masked by the randomness sent from the server. This can be observed by defining the following partial summation, which can be computed by each user in group $l+1$, 
\vspace{-0.1cm}
\begin{align}
    \mathbf{s}^{(l+1)}
    &=\frac{1}{N_{l}}\sum_{i\in{[N_{l}]}}\widetilde{\mathbf{s}}_{i}^{(l)} \notag \\
    &= \frac{1}{N_{l-1}}\sum_{j\in[N_{l-1}]}\widetilde{\mathbf{s}}_{j}^{(l-1)} + \sum_{j\in\mathcal{U}_{l-1}} \mathbf{x}_j^{(l-1)} + \sum_{j\in\mathcal{U}_{l-1}} \mathbf{u}_j^{(l-1)}\label{eq:sum_equality} \\
    &= \mathbf{s}^{(l)} + \sum_{j\in\mathcal{U}_{l-1}} \mathbf{x}_j^{(l-1)} + \sum_{j\in\mathcal{U}_{l-1}} \mathbf{u}_j^{(l-1)}, \label{eq:sum_update_eq}
\end{align} 
where~\eqref{eq:sum_equality} follows from  $\sum_{j\in[N_{l-1}]}\mathbf{r}_{i,j}^{(l-1)} = 0$. 
With the initial partial summation $\mathbf{s}^{(2)}=\frac{1}{N_{1}}\sum_{i\in{[N_{1}]}}\widetilde{\mathbf{s}}_{i}^{(1)}=\mathbf{0}$, 
one can show that $\mathbf{s}^{(l+1)}$ is equal to the aggregation of the models of all surviving users in up to group $l-1$, masked by the randomness sent from the server,
\vspace{-0.06cm}
\begin{equation}\label{eq:prop_partialsum}
    \mathbf{s}^{(l+1)} = \sum_{m\in[l-1]}\sum_{j\in\mathcal{U}_m}\mathbf{x}^{(m)}_j + \sum_{m\in[l-1]}\sum_{j\in\mathcal{U}_m} \mathbf{u}^{(m)}_j. 
\end{equation}
At the final stage, the server obtains the final aggregate value from \eqref{eq:prop_partialsum} and removes the random masks $\sum_{m\in[L]}\sum_{j\in\mathcal{U}_m} \mathbf{u}^{(m)}_j$. 
This approach works well if no user drops out during the execution of the protocol. 
On the other hand, if any user in group $l+1$ drops out, the random vectors masking the models of the $l$-th group in the summation \eqref{eq:sum_equality} cannot be cancelled out.
In the following, we propose a recovery technique that is robust to dropped or delayed users, based on coding theory principles.

\vspace{-0.2cm}\subsection{Adding redundancies to recover the data of dropped or delayed users} \label{sec:coding}
The main intuition behind our recovery strategy is to encode the additive secret shares (masked models) in a way that guarantees secure aggregation when users are dropped or delayed.
To do so, we leverage Lagrange coding~\cite{yu2018lagrange}, which has been applied to other problems such as offloading or collaborative machine learning in the privacy-preserving manner~\cite{so2019codedprivateml,so2020scalable}.
The primary benefits of Lagrange coding over alternative codes that may also be used for introducing redundancy, such as other error-correcting codes, is that Lagrange coding enables us to perform the aggregation operation on the encoded models, and that the final result can be decoded from the computations performed on the encoded models. This is not necessarily true for other error-correcting codes, as they do not guarantee the recovery of the original computation results (i.e., the computations performed on the true values of the model parameters) from the computations performed on the encoded models.
It encodes a given set of $K$ vectors $(\mathbf{v}_1,\ldots, \mathbf{v}_K)$ by using a Lagrange interpolation polynomial. 
One can view this as embedding a given set of vectors on a Lagrange  polynomial, such that each encoded value represents a point on the polynomial.
The resulting encoding enables a set of users to compute a given polynomial function $h$ on the encoded data in a way that any individual computation $\{h(\mathbf{v}_i)\}_{i\in[K]}$ can be reconstructed
using any subset of $deg(h) (K-1) +1$ other computations. 
The reconstruction is done through polynomial interpolation. 
Therefore, one can reconstruct any missing value as long as a sufficient number of other computations  are available, i.e., enough number of points are available to interpolate the polynomial. In our problem of gradient aggregation, the function of interest, $h$, would be linear and accordingly have degree $1$, since it corresponds to the summation of all individual gradient vectors.

{\name} utilizes Lagrange coding for recovery against user dropouts, via a novel strategy that encodes the secret shared values to compute secure aggregation. More specifically, in {\name}, the encoding is performed as follows.
Initially, user $i$ in group $l$ forms a Lagrange interpolation polynomial 
$f^{(l)}_i:\mathbb{F}_q\rightarrow\mathbb{F}_q^{d}$ of degree $N_{l+1}-1$ such that $f^{(l)}_i(\alpha^{(l+1)}_j)=\widetilde{\mathbf{x}}_{i,j}^{(l)}$ for $j\in[N_{l+1}]$, where $\alpha^{(l+1)}_j$ is an evaluation point allocated to user $j$ in group $l+1$. 
This is accomplished by letting 
\begin{equation*} 
f^{(l)}_i(z)=\sum_{j\in[N_{l+1}]}  \widetilde{\mathbf{x}}_{i,j}^{(l)} \cdot \prod_{k\in [N_{l+1}]\setminus\{j\}}\frac{z - \alpha^{(l+1)}_k}{\alpha^{(l+1)}_j - \alpha^{(l+1)}_k}.
\end{equation*}
Then, another set of $N_{l+1}$ distinct evaluation points 
$\{\beta^{(l+1)}_j\}_{j\in[N_{l+1}]}$ are allocated from $\mathbb{F}_q$ such that $\{\beta^{(l+1)}_j\}_{j\in[N_{l+1}]}\cap\{\alpha^{(l+1)}_j\}_{j\in[N_{l+1}]}$ $=\varnothing$. 
Next, user $i\in[N_{l}]$ in group $l$ generates the encoded model,
\begin{equation}\label{eq:LCC_encoding}
    \bar{\mathbf{x}}^{(l)}_{i,j} = f^{(l)}_i(\beta^{(l+1)}_j),
\end{equation}
and sends $\bar{\mathbf{x}}^{(l)}_{i,j}$ to user $j$ in group $(l+1)$. 
%
In addition, user $i\in[N_{l}]$ in group $l$ aggregates the encoded models $\{\bar{\mathbf{x}}_{j,i}^{(l-1)}\}_{j\in \mathcal{U}_{l-1}}$ received from the previous stage, with the partial summation $\mathbf{s}^{(l)}$ from~\eqref{eq:sum_equality} as 
\begin{align}\label{eq:sum_codedmodel}
    \bar{\mathbf{s}}_{i}^{(l)} 
    =\mathbf{s}^{(l)} + \sum_{j\in\mathcal{U}_{l-1}}\bar{\mathbf{x}}_{j,i}^{(l-1)}.
\end{align}
The summation of the masked models in~\eqref{eq:sum_rx_data} and the summation of the coded models in~\eqref{eq:sum_codedmodel} can be viewed as evaluations of a polynomial $g^{(l)}$ such that
\begin{align}\label{eq:polynomial_g}
    \widetilde{\mathbf{s}}_{i}^{(l)} = g^{(l)}(\alpha^{(l)}_i), \\ 
    \bar{\mathbf{s}}_{i}^{(l)} = g^{(l)}(\beta^{(l)}_i),
\end{align}
for $i\in[N_{l}]$, where $g^{(l)}(z)=\mathbf{s}^{(l)}+\sum_{j\in\mathcal{U}_{l-1}} f^{(l-1)}_j(z)$ is a polynomial function with degree at most $N_{l}-1$. 
Then, user $i\in[N_{l}]$ sends the set of messages $\{\widetilde{\mathbf{x}}_{i,j}^{(l)},\bar{\mathbf{x}}_{i,j}^{(l)},\widetilde{\mathbf{s}}_{i}^{(l)}, \bar{\mathbf{s}}_{i}^{(l)}\}$ to user $j$ in group $l+1$.
%
%

Upon receiving the  messages, user $j$ in group $l+1$ reconstructs the missing terms in  $\{\widetilde{\mathbf{s}}_{i}^{(l)}\}_{i\in[N_{l}]}$ (caused by the dropped users in group $l$), computes the partial sum $\mathbf{s}^{(l+1)}$ from~\eqref{eq:sum_equality}, and  updates the terms $\{\widetilde{\mathbf{s}}_{j}^{(l+1)},\bar{\mathbf{s}}_{j}^{(l+1)}\}$ as in~\eqref{eq:sum_rx_data} and~\eqref{eq:sum_codedmodel}.
Users in group $l+1$ can reconstruct each term in  $\{\widetilde{\mathbf{s}}_{i}^{(l)}\}_{i\in[N_{l}]}$
as long as they receive at least $N_{l}$ evaluations out of $2N_{l}$ evaluations from the users in group $l$.
This is because $\{\widetilde{\mathbf{s}}_{i}^{(l)},\bar{\mathbf{s}}_{i}^{(l)}\}_{i\in[N_{l}]}$ are evaluation points of the polynomial $g^{(l)}$ whose degree is at most $N_l-1$.
As a result, the model can be aggregated at each stage as long as at least half of the users at that stage are not dropped. 
As we will demonstrate in the proof of Theorem~\ref{thm:TA_achieve}, as long as the drop rate of the users is below $50\%$, the fraction of dropped users at all stages will be below half with high probability, hence {\name} can proceed with model aggregation at each stage.

\vspace{-0.2cm}
\subsection{Final aggregation and the overall {\name} protocol}\label{sec:proposed_finalaggregation} \vspace{-0.1cm}
For the final aggregation, we need a dummy stage to securely compute the aggregation of all user models, especially for the privacy of the local models of users in group $L$. To do so, we arbitrarily select a set of users who will receive and aggregate the models sent from the users in group $L$. They can be any surviving user who has participated in the protocol, and will be called user $j\in[N_{final}]$ in the final stage, where $N_{final}$ is the number of users selected.

During this phase, users in group $L$ mask their own model with additive secret sharing by using~\eqref{eq:gen_SS}, generate the encoded data by using~\eqref{eq:LCC_encoding}, and aggregate the models received from the users in group $(L-1)$ by using~\eqref{eq:sum_rx_data} and~\eqref{eq:sum_codedmodel}. Then,  user $i$ from group $L$ sends $\{\widetilde{\mathbf{x}}^{(L)}_{i,j}, \bar{\mathbf{x}}^{(L)}_{i,j}, \widetilde{\mathbf{s}}^{(L)}_{i}, \bar{\mathbf{s}}^{(L)}_{i}\}$ to user $j$ in the final stage. 

Upon receiving the set of messages, user $j\in[N_{final}]$ in the final stage recovers the missing terms in $\{\widetilde{\mathbf{s}}^{(L)}_i\}_{i\in[N_L]}$, and aggregates them with the masked models, 
\begin{align}
    \widetilde{\mathbf{s}}_{j}^{(final)} = \frac{1}{N_{L}}\sum_{i\in[N_{L}]}\widetilde{\mathbf{s}}_{i}^{(L)} + \sum_{i\in\mathcal{U}_{L}}\widetilde{\mathbf{x}}_{i,j}^{(L)} \label{eq:finalaggr},\\ 
    \bar{\mathbf{s}}_{j}^{(final)} = 
    \frac{1}{N_{L}}\sum_{i\in[N_{L}]}\widetilde{\mathbf{s}}_{i}^{(L)} + \sum_{i\in\mathcal{U}_{L}}\bar{\mathbf{x}}_{i,j}^{(L)} , \label{eq:finalaggr_LCC} 
\end{align} 
and sends the resulting $\{\widetilde{\mathbf{s}}_{j}^{(final)}, \bar{\mathbf{s}}_{j}^{(final)}\}$ to the server. 

The server then recovers the summations $\{\widetilde{\mathbf{s}}^{(final)}_j\}_{j\in[N_{final}]}$, by reconstructing any missing terms in~\eqref{eq:finalaggr} using the set of received values~\eqref{eq:finalaggr} and~\eqref{eq:finalaggr_LCC}. 
Finally, the server computes the average of the summations from \eqref{eq:finalaggr} and removes the random masks $\sum_{m\in[L]}\sum_{j\in\mathcal{U}_m}\mathbf{u}^{(m)}_j$ from the aggregate, which, as can be observed  from~\eqref{eq:sum_equality}-\eqref{eq:prop_partialsum},  is equal to the aggregate of the individual models of all surviving users, 
\begin{equation}\label{eq:final_aggregation}
    \frac{1}{N_{final}}\sum_{j\in [N_{final}]}\widetilde{\mathbf{s}}_j^{(final)}  -  \sum_{m\in[L]}\sum_{j\in\mathcal{U}_m}\mathbf{u}^{(m)}_j= \sum_{m\in[L]}\sum_{j\in\mathcal{U}_m}\mathbf{x}^{(m)}_j.
\end{equation} 
Having all above steps, the overall  {\name} protocol is presented in Algorithm~\ref{Alg}.

\begin{algorithm}[t!]
\small
  \caption{{\name}}\label{Alg} 
  \begin{algorithmic}[1]
    \INPUT{Local models $\mathbf{x}_i^{(l)}$ of users $i\in[N_l]$ in group $l\in[L]$.} \
    \OUTPUT{Aggregated model $\sum_{l\in[L]}\sum_{i\in\mathcal{U}_l} \mathbf{x}_i^{(l)}$.} \ 
    \vspace{0.1cm} 
    \FOR{group $l=1,\ldots,L$}
    \FOR{user $i=1,\ldots,N_l$}
    \STATE Compute the masked model $\{\widetilde{\mathbf{x}}_{i,j}^{(l)}\}_{l\in[N_{l+1}]}$ from \eqref{eq:gen_SS}. 
    \STATE Generate the encoded model $\{\bar{\mathbf{x}}_{i,j}^{(l)}\}_{j\in[N_{l+1}]}$ from \eqref{eq:LCC_encoding}. 

    \vspace{0.1cm}\IF{$l=1$} 
    
    \STATE Initialize  $\widetilde{\mathbf{s}}_i^{(1)}=\bar{\mathbf{s}}_i^{(1)} = \mathbf{0}$. 
    \ELSE
    \STATE Reconstruct the missing values in  $\{\widetilde{\mathbf{s}}_k^{(l-1)}\}_{k\in[N_{l-1}]}$ due to the dropped users in group $l-1$. 
    \STATE Update the aggregate value $\widetilde{\mathbf{s}}_i^{(l)}$ from  \eqref{eq:sum_rx_data}. 
    \STATE Compute the coded aggregate value $\bar{\mathbf{s}}_i^{(l)}$ from  \eqref{eq:sum_codedmodel}.     
    \ENDIF
    \STATE Send $\{\widetilde{\mathbf{x}}_{i,j}^{(l)}, \bar{\mathbf{x}}_{i,j}^{(l)},  \widetilde{\mathbf{s}}_i^{(l)}, \bar{\mathbf{s}}_i^{(l)}\}$ to user $j\in[N_{l+1}]$ in group $l+1$ ($j\in[N_{final}]$ if $l=L$).   
    \ENDFOR
    
    \ENDFOR

    \vspace{0.1cm}
    \FOR{user $i=1,\ldots,N_{final}$}
    \STATE Reconstruct the missing values in  $\{\widetilde{\mathbf{s}}_k^{(L)}\}_{k\in[N_L]}$ due to the dropped users in group $L$. 
    \STATE Compute  $\widetilde{\mathbf{s}}_i^{(final)}$ from \eqref{eq:finalaggr} and $\bar{\mathbf{s}}_i^{(final)}$ from \eqref{eq:finalaggr_LCC}.
    
    \STATE Send $\{\widetilde{\mathbf{s}}_i^{(final)}, \bar{\mathbf{s}}_i^{(final)}\}$ to the server. 
    \ENDFOR

    \STATE Server computes the final aggregated model from \eqref{eq:final_aggregation}.

  \end{algorithmic}
\end{algorithm}

\section{Theoretical Guarantees of {\name}}~\label{5-MainResults}
In this section, we formally state our main theoretical result.
\begin{theorem} \label{thm:TA_achieve}
    {\name} can simultaneously achieve: 
    \begin{enumerate}
        \item robustness guarantee to any user dropout rate $p<0.5$, with probability approaching to 1 as the number of users $N \rightarrow \infty$,
        
        \item privacy guarantee against up to $T=(0.5-\epsilon)N$ colluding users,  with probability approaching to 1 as the number of users $N \rightarrow \infty$, and for any $\epsilon>0$,
        
        \item aggregation overhead of $C=O(N\log{N})$.
    \end{enumerate}
\end{theorem}

\begin{remark}\label{rmk:tradeoff_robustness_privacy} \normalfont
    Theorem~\ref{thm:TA_achieve} states that {\name} can tolerate up to $50\%$ user dropout rate and $\frac{N}{2}$ collusions between the users, simultaneously.
    {\name} can guarantee robustness against an even higher number of user dropouts by sacrificing the privacy guarantee as a trade-off. 
    Specifically, when we generate and communicate $k$ set of evaluation points during Lagrange coding, we can recover the partial aggregations by decoding the polynomial in~\eqref{eq:polynomial_g} as long as each user receives $N_l$ evaluations, i.e., $(1+k)(N_l - p N_l) \geq N_{l}$. 
    As a result, {\name} can tolerate up to a $p < \frac{k}{1+k}$ user dropout rate. 
    On the other hand, the individual models will be revealed whenever $T(k+1) \geq N$.
    In this case, one can guarantee privacy against up to $(\frac{1}{k+1}-\epsilon)N$ colluding users for any $\epsilon > 0$. 
    This demonstrates a trade-off between robustness and privacy guarantees achieved by {\name}, that is, one can increase the robustness guarantee by reducing the privacy guarantee and vice versa. 
\end{remark}

\begin{proof} 
The proof of Theorem~\ref{thm:TA_achieve} is presented in Appendix~C.
\end{proof}

As we showed in the proof of Theorem $1$, {\name} achieves its robustness and privacy guarantees by choosing a group size of  $N_l=\frac{1}{c}\log{N}$ for all $l\in[L]$ where $c\triangleq \min\{D(0.5||p), D(0.5||\frac{T}{N})\}$ and $D(a||b)$ is the Kullback-Leibler (KL) distance between two Bernoulli distributions with parameter $a$ and $b$~\cite{cover2006}.
We can further reduce the aggregation overhead if we choose a smaller group size $N_l$.
However, we cannot further reduce the group size beyond $O(\log{N})$ because when $0<D(0.5||p)<1$ ($\frac{1}{c}>1$) and $N_l = \log{N}$, the probability that {\name} guarantees the accuracy of full model aggregation goes to $0$ with sufficiently large number of users, which is stated in Theorem~\ref{thm:converse_droprate}.

\begin{theorem}{(Converse)}\label{thm:converse_droprate}
    When $0<D(0.5||p)<1$ and $\Nl = \log{N}$ for all $l\in[L]$, the probability that {\name} achieves the robustness guarantee to any user dropout rate $p<0.5$ goes to $0$ as the number of users $N\rightarrow \infty$.
\end{theorem}

\begin{proof} \normalfont 
    The proof of Theorem~\ref{thm:converse_droprate} is presented in Appendix~D.
\end{proof}

\subsection{Generalized {\name}}
    Theorem~\ref{thm:TA_achieve} states that the privacy of each \emph{individual model} is guaranteed against any collusion between the server and up to $\frac{N}{2}$ users.
    On the other hand, a collusion between the server and a subset of users can reveal the \emph{partial aggregation} of a group of honest users. 
    For instance, 
    a collusion between the server and a user in group $l$ can reveal the  partial aggregation of the models of all users up to group $l-2$, as the colluding server can remove the random masks in~\eqref{eq:prop_partialsum}. 
    However, the privacy protection can be strengthened to guarantee the privacy of \emph{any} partial aggregation, i.e., the aggregate of any subset of user models, with a simple modification.


The modified protocol follows the same steps in  Algorithm~\ref{Alg} except that the random mask $\mathbf{u}_i^{(l)}$ in~\eqref{eq:gen_SS} is generated by each user individually, instead of being generated by the server. At the end of the aggregation phase, the server learns  $\sum_{m\in[L]}\sum_{j\in\mathcal{U}_m}(\mathbf{x}^{(m)}_j+\mathbf{u}^{(m)}_j)$. 
Simultaneously, the protocol executes an additional random partitioning strategy to aggregate the random masks  $\mathbf{u}^{(m)}_j$, at the end of which the server obtains  $\sum_{m\in[L]}\sum_{j\in\mathcal{U}_m} \mathbf{u}^{(m)}_j$ and recovers $\sum_{m\in[L]}\sum_{j\in\mathcal{U}_m}\mathbf{x}^{(m)}_j$. 
In this second partitioning, $N$ users are randomly allocated into $L$ groups with a group size of $N_{l}$. User $i$ in group $l'\in[L]$ then secret shares ${\mathbf{u}}^{(l')}_i$ with the users in group $l'+1$, by generating and sending a secret share denoted by $[{\mathbf{u}}^{(l')}_i]_j$ to user $j$ in group $l'+1$. For secret sharing, we utilize Shamir's  $\frac{N_{l}}{2}$-out-of-$N_l$ secret sharing protocol~\cite{shamir1979share}. 
Let $\mathcal{U}^{'}_{l'}$ denote the surviving users in group $l'$ in the second partitioning. 
User $i$ in group $l'$ then aggregates the received secret shares $\sum_{j\in\mathcal{U}^{'}_{l'-1}}{\big[\mathbf{u}}^{(l'-1)}_j\big]_i$, which in turn is a secret share of  $\sum_{j\in\mathcal{U}^{'}_{l'-1}}{\mathbf{u}}^{(l'-1)}_j$, and sends the sum to the server. 
Finally, the server reconstructs $\sum_{j\in\mathcal{U}^{'}_{l'}}{\mathbf{u}}^{(l')}_j$ for all $l'\in[L]$ and recovers the aggregate of the individual models of all surviving users by subtracting $\big\{ \sum_{j\in\mathcal{U}^{'}_{l'}}{\mathbf{u}}^{(l')}_j \big\}_{l'\in[L]}$ from the aggregate $\sum_{m\in[L]}\sum_{j\in\mathcal{U}_m}(\mathbf{x}^{(m)}_j+\mathbf{u}^{(m)}_j)$. 

In this generalized version of {\name}, the privacy of any partial aggregation, i.e., the aggregate of any subset of user models, can be protected as long as a  collusion between the server and the users does not reveal the aggregation of the random masks, $\sum_{j\in\mathcal{U}_{l}}{\mathbf{u}}^{(l)}_j$ in \eqref{eq:prop_partialsum} for any $l\in[L]$. 
Since there are at least $\frac{N}{2}$ unknown random masks generated by honest users and the server only knows $L=\frac{N}{N_l}$ equations, i.e., $\big\{ \sum_{j\in\mathcal{U}^{'}_{l}}{\mathbf{u}}^{(l)}_j \big\}_{l\in[L]}$,
the server cannot calculate $\sum_{j\in\mathcal{U}_{l}}{\mathbf{u}}^{(l)}_j$ for any $l\in[L]$. 
Therefore, a collusion between the server and users cannot reveal the partial aggregate as they cannot remove the random masks in~\eqref{eq:prop_partialsum}.
We now formally state the privacy guarantee, robustness guarantee, and  aggregation overhead of the generalized {\name} protocol in  Theorem~\ref{col:modifiedTA_achieve}.

\vspace{0.2cm}
\begin{theorem}~\label{col:modifiedTA_achieve}
Generalized {\name} simultaneously achieves 1), 2), and 3) from  Theorem~\ref{thm:TA_achieve}. In addition, it provides privacy guarantee for the partial aggregate of any subset of user models, against any collusion between the server and up to $T=(0.5-\epsilon)N$ users for any $\epsilon>0$, with probability approaching to $1$ as the number of users $N\rightarrow \infty$. 
\end{theorem}
\begin{proof}
    The proof of Theorem~\ref{col:modifiedTA_achieve} is presented in Appendix~E.
\end{proof}

\section{Experiments}~\label{6-Experimental Evaluation}
In this section, we evaluate the performance of {\name} by experiments over up to $N=200$ users for various user dropout rates and bandwidth conditions.
%
\subsection{Experiment setup}
\noindent \textbf{Platform.} 
In our experiments, we implement {\name} on a distributed platform by using \texttt{FedML} library \cite{he2020fedml}, and examine its total running time with respect to the state-of-the-art~\cite{bonawitz2017practical}.
Computation is performed in a distributed network over the Amazon EC2 cloud using \texttt{m3.medium} machine instances. 
Communication is implemented using the {\tt MPI4Py}~\cite{dalcin2011parallel} message passing interface on {\tt Python}.
The default setting for the maximum bandwidth constraint of \texttt{m3.medium} machine instances is $1 Gbps$.
The model size, $d$, is fixed to $100,\!000$ with $32$ bit entries, and the field size, $q$, is set as the largest prime within $32$ bits.
We summarize the simulation parameters in Table~\ref{tbl:simul_params}.

\begin{table}[t!]
\small
\caption{Summary of simulation parameters. }
\label{tbl:simul_params}
\begin{center}
\begin{tabular}{ccc}
\toprule
Variable & Definition & Value \\
\midrule
N & \hspace{-0.4cm} number of users & 4 $\sim$ 200 \\
d & \hspace{-0.4cm} model size (32 bit entries) & 100000 \\
p & \hspace{-0.4cm} dropout rate & 10\%, 30\%, 50\%  \\
q & \hspace{-0.4cm} field size & $2^{32} - 5$ \\
  & \hspace{-0.4cm} maximum bandwidth constraint & 100Mbps $\sim$ 1Gbps  \\
\bottomrule
\end{tabular}
\end{center}
\end{table}

\begin{figure}%
    \centering
    \subfloat[
    {\name}. Each arrow is carried out sequentially. {\name} requires $7$ stages.
    ]{{\includegraphics[width=0.8\linewidth]{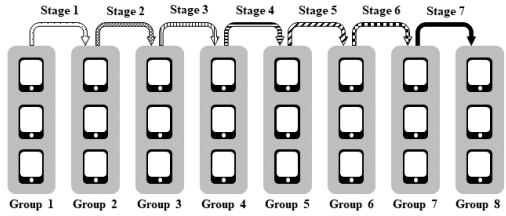}
    \label{fig:TA_and_TA+} }}%
    \qquad
    \subfloat[
    {\name}$+$. Arrows for the same execution stage are carried out simultaneously. {\name}$+$ requires only $3$ stages.
    ]{{\includegraphics[width=0.8\linewidth]{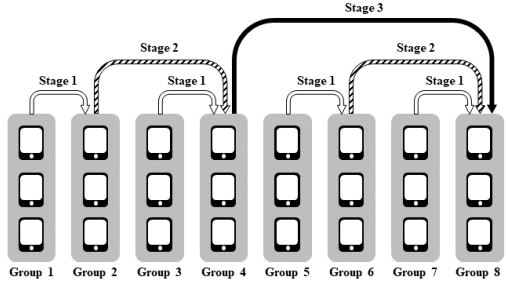} \label{fig:TA++}}}%
    \caption{
    Example networks with $N=24$, $N_l=3$ and $L=8$.
    An arrow represents that users in one group generate and send messages to the users in the next group.
    }%
    \label{fig:network_example}%
\vspace{-0.2cm}
\end{figure}
\begin{figure}
\centering
\includegraphics[width=0.85\linewidth]{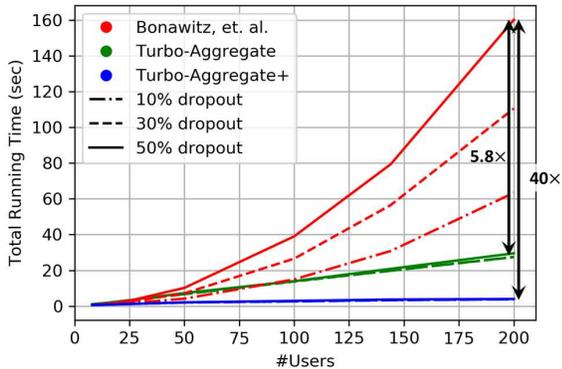}
\vspace{-0.25cm}\caption{Total running time of {\name} versus the benchmark protocol  \cite{bonawitz2017practical} as the number of users increases, for various user dropout rates.}
\label{fig:trainingtimeSec5}
\vspace{-0.4cm}\end{figure}

\noindent \textbf{Modeling user dropouts.}
To model the dropped users in {\name}, we randomly select $pN_{l}$ users out of $N_l$ users in group $l\in[L]$ where $p$ is the  dropout rate. 
We consider the worst case scenario where the selected users drop after receiving the messages sent from the previous group (users in group $l-1$) and do not send their messages to users in group $l+1$. 
To model the dropped users in the benchmark protocol, we follow the scenario in~\cite{bonawitz2017practical}. We randomly select $pN$ users out of $N$ users, which artificially drop after sending their masked models. In this case, the server has to reconstruct the pairwise seeds of the dropped users and execute a pseudo random generator using the reconstructed seeds to remove the random masks (the details are provided  in  Appendix~A).

\noindent \textbf{Implemented Schemes.}
We implement the following schemes for performance evaluation.
{\color{Mycolor3} For the schemes with {\name}, we use $\Nl=\log N$.}
\begin{enumerate}
    \item \textbf{{\name}}: 
    For our first implementation, we directly implement {\name} as described in Section~\ref{3-Algorithm}, where the $L$ execution stages are performed sequentially.

    \item \textbf{{\name}$+$}: We can speed up {\name} by parallelizing the $L$ execution stages. 
    To do so, we again utilize the circular aggregation topology but leverage a tree structure for flooding the information between different groups across the network, which reduces the required number of execution stages from $L-1$ to $\log{L}$. We refer to this implementation as {\name}$+$. 
    Figure~\ref{fig:network_example} demonstrates the difference between {\name}$+$ and {\name} through an example network of $N=24$ users and $L=8$ groups.
    {\name}$+$ requires only $3$ stages to complete the protocol while {\name} carries out each execution stage sequentially and requires $7$ stages.

    \item \textbf{Benchmark}: 
    We implement the benchmark protocol \cite{bonawitz2017practical} where a server mediates the communication between users to exchange the  information required for key agreements (rounds of advertising and sharing keys) and users send their masked models to the server (masked input collection). 
    One can also speed up the rounds of advertising and sharing keys by allowing users to communicate in parallel. 
    However, this has minimal effect on the total running time of the protocol, as the total running time is dominated by the overhead when the server generates the pairwise masks~\cite{bonawitz2017practical}. 
\end{enumerate}

\subsection{Performance evaluation}

For performance analysis, we measure the total running time for a single round of secure aggregation with each protocol while increasing the number of users $N$ gradually for different user dropout rates. 
We use synthesized vectors for locally trained models and do not include the local training time in the total running time. One can also consider the entire learning process and since all other steps remain the same for the three schemes, we expect the same speedup in the aggregation phase.
Our results are demonstrated in Figure~\ref{fig:trainingtimeSec5}. We make the following key observations.

\begin{itemize}
    \item Total running time of {\name} and {\name}$+$ are almost linear in the number of users, while for the benchmark protocol, the total running time is quadratic in the number of users.

    \item {\name} and {\name}$+$ provide a stable total running time as the user dropout rate increases. This is because the encoding and decoding time of {\name} do not change significantly when the dropout rate increases, and we do not require additional information to be transmitted from the remaining users when some users are dropped or delayed.
    On the other hand, for the benchmark protocol, the running time significantly increases as the dropout rate increases. This is because the total running time is dominated by the reconstruction of pairwise masks at the server, which substantially increases as the number of dropped users increases. 
    \item {\name} and {\name}$+$ provide a speedup of up to $5.8\times$ and $40\times$ over the benchmark, respectively, for a user dropout rate of up to $50\%$ with $N=200$ users. 
    This gain is expected to increase further as the number of users increases. 
\end{itemize} 

To illustrate the impact of user dropouts, we present the breakdown of the total running time of the three schemes and the corresponding observations in Appendix~F1.
We further study the impact of the bandwidth, by measuring the total running time with various communication bandwidth constraints.
{\name} provides substantial gain over the state-of-the-art in environments with more severe bandwidth constraints.
The details of these additional experiments are presented in Appendix~F2.

In this section, we have primarily focused on the aggregation phase and measured a single round of the secure aggregation phase with synthesized vectors for the locally trained models.
This is due to the fact that these vectors can be replaced with any trained model using the real world federated learning setups. 
We further investigate the performance of {\name} in real world federated learning setups by implementing both training phase and aggregation phase.
{\name} still provides substantial speedup over the benchmark, which is detailed in Appendix F3.

\section{Conclusion}~\label{7-Conclusion}
This paper presents the first secure aggregation framework that theoretically achieves an aggregation overhead of $O(N\log{N})$ in a network with $N$ users, as opposed to the prior $O(N^2)$ overhead, while tolerating up to a user dropout rate of $50\%$. Furthermore, via experiments over Amazon EC2, we demonstrated that {\name} achieves a total running time that grows almost linearly in the number of users, and provides up to
$40\times$ speedup over the state-of-the-art scheme with $N=200$ users. 

{\name} is particularly suitable for wireless topologies, in which network conditions and user availability can vary rapidly, as {\name} can provide a resilient framework to handle such unreliable network conditions. 
Specifically, if some users cause unexpected delays due to unstable connection, {\name} can simply treat them as user dropouts and can reconstruct the information of dropped or delayed users in the previous groups as long as half of the users remain.
One may also leverage the geographic heterogeneity of wireless networks to better form the communication groups in {\name}. An interesting future direction would be to explore how to optimize the multi-group communication structure of {\name} based on the specific topology of the users, as well as the network conditions.

In this work, we have focused on protecting the privacy of individual models against an  honest-but-curious server and up to $T$ colluding users so that no information is revealed about the individual models beyond their aggregated value.
If one would like to further limit the information that may be revealed from the aggregated model, differential privacy can be utilized to ensure that the individual data points cannot be identified from the aggregated model.
All the benefits of differential privacy could be applied to our approach by adding noise to the local models before the aggregation phase in Turbo-Aggregate. Combining these two techniques is another interesting future direction. 

Finally, the implementation of Turbo-Aggregate in a real-world large-scale distributed system would be another interesting future direction. This would require addressing the following three challenges. 
First, the computation complexity of implementing the random grouping strategy may increase as the number of users increases. 
Second, Turbo-Aggregate currently focuses on protecting the privacy against honest-but-curious adversaries. In settings  with malicious (Byzantine) adversaries who wish to manipulate the global model by poisoning their local datasets, one may require additional strategies to protect the resilience of the trained model.  
One approach is combining secure aggregation with an outlier detection algorithm as proposed in \cite{so2020byzantine}, which has a communication cost of $O(N^2)$ that limits its scalability to large federated learning systems. It would be an interesting direction to leverage Turbo-Aggregate to address this challenge, i.e., develop a communication-efficient secure aggregation strategy against Byzantine adversaries. 
Third, communication may still be a bottleneck in severely resource-constrained systems since users need to exchange the masked models with each other, whose size is as large as the size of the global model.
To overcome this bottleneck, one may leverage model compression techniques or group knowledge transfer~\cite{he2020group}.

\nocite{Wassily1963Bound,kedlaya2011fast}
\bibliographystyle{IEEEtran}
\bibliography{Main.bbl}

\begin{thebibliography}{10}
\providecommand{\url}[1]{#1}
\csname url@samestyle\endcsname
\providecommand{\newblock}{\relax}
\providecommand{\bibinfo}[2]{#2}
\providecommand{\BIBentrySTDinterwordspacing}{\spaceskip=0pt\relax}
\providecommand{\BIBentryALTinterwordstretchfactor}{4}
\providecommand{\BIBentryALTinterwordspacing}{\spaceskip=\fontdimen2\font plus
\BIBentryALTinterwordstretchfactor\fontdimen3\font minus
  \fontdimen4\font\relax}
\providecommand{\BIBforeignlanguage}[2]{{%
\expandafter\ifx\csname l@#1\endcsname\relax
\typeout{** WARNING: IEEEtran.bst: No hyphenation pattern has been}%
\typeout{** loaded for the language `#1'. Using the pattern for}%
\typeout{** the default language instead.}%
\else
\language=\csname l@#1\endcsname
\fi
#2}}
\providecommand{\BIBdecl}{\relax}
\BIBdecl

\bibitem{so2021turbo}
J.~So, B.~G{\"u}ler, and A.~S. Avestimehr, ``Turbo-aggregate: Breaking the
  quadratic aggregation barrier in secure federated learning,'' \emph{IEEE
  Journal on Selected Areas in Information Theory}, 2021.

\bibitem{mcmahan2016communication}
H.~B. McMahan, E.~Moore, D.~Ramage, S.~Hampson, and B.~A. y~Arcas,
  ``Communication-efficient learning of deep networks from decentralized
  data,'' in \emph{Int. Conf. on Artificial Int. and Stat. (AISTATS)}, 2017,
  pp. 1273--1282.

\bibitem{bonawitz2016practical}
K.~Bonawitz, V.~Ivanov, B.~Kreuter, A.~Marcedone, H.~B. McMahan, S.~Patel,
  D.~Ramage, A.~Segal, and K.~Seth, ``Practical secure aggregation for
  federated learning on user-held data,'' \emph{Conference on Neural
  Information Processing Systems}, 2016.

\bibitem{bonawitz2017practical}
------, ``Practical secure aggregation for privacy-preserving machine
  learning,'' in \emph{Proceedings of the 2017 ACM SIGSAC Conference on
  Computer and Communications Security}, 2017, pp. 1175--1191.

\bibitem{kairouz2019advances}
P.~Kairouz, H.~B. McMahan, B.~Avent, A.~Bellet, M.~Bennis, A.~N. Bhagoji,
  K.~Bonawitz, Z.~Charles, G.~Cormode, R.~Cummings \emph{et~al.}, ``Advances
  and open problems in federated learning,'' \emph{arXiv preprint
  arXiv:1912.04977}, 2019.

\bibitem{NIPS2019_9617}
L.~Zhu, Z.~Liu, and S.~Han, ``Deep leakage from gradients,'' in \emph{Advances
  in Neural Information Processing Systems 32}, 2019, pp. 14\,774--14\,784.

\bibitem{8737416}
Z.~{Wang}, M.~{Song}, Z.~{Zhang}, Y.~{Song}, Q.~{Wang}, and H.~{Qi}, ``Beyond
  inferring class representatives: User-level privacy leakage from federated
  learning,'' in \emph{IEEE INFOCOM 2019 - IEEE Conference on Computer
  Communications}, 2019, pp. 2512--2520.

\bibitem{geiping2020inverting}
J.~Geiping, H.~Bauermeister, H.~Dröge, and M.~Moeller, ``Inverting gradients
  -- how easy is it to break privacy in federated learning?'' \emph{arXiv
  preprint arXiv:2003.14053}, 2020.

\bibitem{yang2018applied}
T.~Yang, G.~Andrew, H.~Eichner, H.~Sun, W.~Li, N.~Kong, D.~Ramage, and
  F.~Beaufays, ``Applied federated learning: Improving google keyboard query
  suggestions,'' \emph{arXiv preprint arXiv:1812.02903}, 2018.

\bibitem{mcmahan2017learning}
H.~B. McMahan, D.~Ramage, K.~Talwar, and L.~Zhang, ``Learning differentially
  private recurrent language models,'' \emph{Int. Conf. on Learning
  Representations (ICLR)}, 2018.

\bibitem{bonawitz2019towards}
K.~Bonawitz, H.~Eichner, W.~Grieskamp, D.~Huba, A.~Ingerman, V.~Ivanov,
  C.~Kiddon, J.~Konecny, S.~Mazzocchi, H.~B. McMahan \emph{et~al.}, ``Towards
  federated learning at scale: System design,'' in \emph{2nd SysML Conf.},
  2019.

\bibitem{li2020federated}
T.~Li, A.~K. Sahu, A.~Talwalkar, and V.~Smith, ``Federated learning:
  Challenges, methods, and future directions,'' \emph{IEEE Signal Processing
  Magazine}, vol.~37, no.~3, pp. 50--60, 2020.

\bibitem{beimel2011secret}
A.~Beimel, ``Secret-sharing schemes: a survey,'' in \emph{Int. Conf. on Coding
  and Cryptology}.\hskip 1em plus 0.5em minus 0.4em\relax Springer, 2011, pp.
  11--46.

\bibitem{evans2018pragmatic}
D.~Evans, V.~Kolesnikov, M.~Rosulek \emph{et~al.}, ``A pragmatic introduction
  to secure multi-party computation,'' \emph{Foundations and Trends in Priv.
  and Sec.}, vol.~2, no. 2-3, pp. 70--246, 2018.

\bibitem{yu2018lagrange}
Q.~Yu, S.~Li, N.~Raviv, S.~M.~M. Kalan, M.~Soltanolkotabi, and A.~S.
  Avestimehr, ``Lagrange coded computing: Optimal design for resiliency,
  security and privacy,'' in \emph{Int. Conf. on Artificial Int. and Stat.
  (AISTATS)}, 2019.

\bibitem{lalitha2019peer}
A.~Lalitha, O.~C. Kilinc, T.~Javidi, and F.~Koushanfar, ``Peer-to-peer
  federated learning on graphs,'' \emph{arXiv preprint arXiv:1901.11173}, 2019.

\bibitem{he2019central}
C.~He, C.~Tan, H.~Tang, S.~Qiu, and J.~Liu, ``Central server free federated
  learning over single-sided trust social networks,'' \emph{arXiv preprint
  arXiv:1910.04956}, 2019.

\bibitem{yao1982protocols}
A.~C. Yao, ``Protocols for secure computations,'' in \emph{IEEE Ann. Symp. on
  Foundations of Comp. Sci.}, 1982, pp. 160--164.

\bibitem{shamir1979share}
A.~Shamir, ``How to share a secret,'' \emph{Communications of the ACM},
  vol.~22, no.~11, pp. 612--613, 1979.

\bibitem{ben1988completeness}
M.~Ben-Or, S.~Goldwasser, and A.~Wigderson, ``Completeness theorems for
  non-cryptographic fault-tolerant distributed computation,'' in \emph{ACM
  Symposium on Theory of Computing}.\hskip 1em plus 0.5em minus 0.4em\relax
  ACM, 1988, pp. 1--10.

\bibitem{beerliova2008perfectly}
Z.~Beerliov{\'a}-Trub{\'\i}niov{\'a} and M.~Hirt, ``Perfectly-secure {MPC} with
  linear communication complexity,'' in \emph{Theory of Cryptography
  Conference}.\hskip 1em plus 0.5em minus 0.4em\relax Springer, 2008, pp.
  213--230.

\bibitem{burkhart2010sepia}
M.~Burkhart, M.~Strasser, D.~Many, and X.~Dimitropoulos, ``Sepia:
  Privacy-preserving aggregation of multi-domain network events and
  statistics,'' \emph{Network}, vol.~1, no. 101101, 2010.

\bibitem{leontiadis2014private}
I.~Leontiadis, K.~Elkhiyaoui, and R.~Molva, ``Private and dynamic time-series
  data aggregation with trust relaxation,'' in \emph{International Conference
  on Cryptology and Network Security}.\hskip 1em plus 0.5em minus 0.4em\relax
  Springer, 2014, pp. 305--320.

\bibitem{rastogi2010differentially}
V.~Rastogi and S.~Nath, ``Differentially private aggregation of distributed
  time-series with transformation and encryption,'' in \emph{ACM SIGMOD Int.
  Conf. on Management of data}, 2010, pp. 735--746.

\bibitem{halevi2011secure}
S.~Halevi, Y.~Lindell, and B.~Pinkas, ``Secure computation on the web:
  Computing without simultaneous interaction,'' in \emph{Annual Cryptology
  Conf.}\hskip 1em plus 0.5em minus 0.4em\relax Springer, 2011, pp. 132--150.

\bibitem{gentry2009fully}
C.~Gentry and D.~Boneh, \emph{A fully homomorphic encryption scheme}.\hskip 1em
  plus 0.5em minus 0.4em\relax Stanford University, Stanford, 2009, vol.~20,
  no.~09.

\bibitem{damgaard2012multiparty}
I.~Damg{\aa}rd, V.~Pastro, N.~Smart, and S.~Zakarias, ``Multiparty computation
  from somewhat homomorphic encryption,'' in \emph{Annual Cryptology
  Conf.}\hskip 1em plus 0.5em minus 0.4em\relax Springer, 2012, pp. 643--662.

\bibitem{dwork2006calibrating}
C.~Dwork, F.~McSherry, K.~Nissim, and A.~Smith, ``Calibrating noise to
  sensitivity in private data analysis,'' in \emph{Theory of Crypto.
  Conf.}\hskip 1em plus 0.5em minus 0.4em\relax Springer, 2006, pp. 265--284.

\bibitem{geyer2017differentially}
R.~C. Geyer, T.~Klein, and M.~Nabi, ``Differentially private federated
  learning: A client level perspective,'' \emph{arXiv preprint
  arXiv:1712.07557}, 2017.

\bibitem{wei2020federated}
K.~Wei, J.~Li, M.~Ding, C.~Ma, H.~H. Yang, F.~Farokhi, S.~Jin, T.~Q. Quek, and
  H.~V. Poor, ``Federated learning with differential privacy: Algorithms and
  performance analysis,'' \emph{IEEE Transactions on Information Forensics and
  Security}, 2020.

\bibitem{acs2011have}
G.~{\'A}cs and C.~Castelluccia, ``I have a dream! (differentially private smart
  metering),'' in \emph{International Workshop on Information Hiding}.\hskip
  1em plus 0.5em minus 0.4em\relax Springer, 2011, pp. 118--132.

\bibitem{konevcny2016federated}
J.~Konečný, H.~B. McMahan, F.~X. Yu, P.~Richtarik, A.~T. Suresh, and
  D.~Bacon, ``Federated learning: Strategies for improving communication
  efficiency,'' in \emph{Conference on Neural Information Processing Systems:
  Workshop on Private Multi-Party Machine Learning}, 2016.

\bibitem{bonawitz2019federated}
K.~Bonawitz, F.~Salehi, J.~Kone{\v{c}}n{\`y}, B.~McMahan, and M.~Gruteser,
  ``Federated learning with autotuned communication-efficient secure
  aggregation,'' \emph{arXiv preprint arXiv:1912.00131}, 2019.

\bibitem{he2018cola}
L.~He, A.~Bian, and M.~Jaggi, ``Cola: Decentralized linear learning,'' in
  \emph{Advances in Neural Information Processing Systems}, 2018, pp.
  4536--4546.

\bibitem{lian2017can}
X.~Lian, C.~Zhang, H.~Zhang, C.-J. Hsieh, W.~Zhang, and J.~Liu, ``Can
  decentralized algorithms outperform centralized algorithms? {A} case study
  for decentralized parallel stochastic gradient descent,'' in \emph{Advances
  in Neural Information Processing Systems}, 2017, pp. 5330--5340.

\bibitem{liu2019edge}
L.~Liu, J.~Zhang, S.~Song, and K.~B. Letaief, ``Edge-assisted hierarchical
  federated learning with non-iid data,'' \emph{arXiv preprint
  arXiv:1905.06641}, 2019.

\bibitem{li2018pipe}
Y.~Li, M.~Yu, S.~Li, S.~Avestimehr, N.~S. Kim, and A.~Schwing, ``Pipe-sgd: A
  decentralized pipelined sgd framework for distributed deep net training,'' in
  \emph{Advances in Neural Information Processing Systems}, 2018, pp.
  8045--8056.

\bibitem{eichner2019semi}
H.~Eichner, T.~Koren, H.~B. McMahan, N.~Srebro, and K.~Talwar, ``Semi-cyclic
  stochastic gradient descent,'' \emph{arXiv preprint arXiv:1904.10120}, 2019.

\bibitem{bhagoji2018analyzing}
A.~N. Bhagoji, S.~Chakraborty, P.~Mittal, and S.~Calo, ``Analyzing federated
  learning through an adversarial lens,'' \emph{arXiv preprint
  arXiv:1811.12470}, 2018.

\bibitem{fang2019local}
M.~Fang, X.~Cao, J.~Jia, and N.~Z. Gong, ``Local model poisoning attacks to
  byzantine-robust federated learning,'' \emph{arXiv preprint
  arXiv:1911.11815}, 2019.

\bibitem{sun2019can}
Z.~Sun, P.~Kairouz, A.~T. Suresh, and H.~B. McMahan, ``Can you really backdoor
  federated learning?'' \emph{arXiv preprint arXiv:1911.07963}, 2019.

\bibitem{smith2017federated}
V.~Smith, C.-K. Chiang, M.~Sanjabi, and A.~S. Talwalkar, ``Federated multi-task
  learning,'' in \emph{Advances in Neural Information Processing Systems},
  2017, pp. 4424--4434.

\bibitem{mohri2019agnostic}
M.~Mohri, G.~Sivek, and A.~T. Suresh, ``Agnostic federated learning,''
  \emph{arXiv preprint arXiv:1902.00146}, 2019.

\bibitem{li2019fair}
T.~Li, M.~Sanjabi, and V.~Smith, ``Fair resource allocation in federated
  learning,'' \emph{arXiv preprint arXiv:1905.10497}, 2019.

\bibitem{li2019convergence}
X.~Li, K.~Huang, W.~Yang, S.~Wang, and Z.~Zhang, ``On the convergence of fedavg
  on non-iid data,'' \emph{arXiv preprint arXiv:1907.02189}, 2019.

\bibitem{syta2017scalable}
E.~Syta, P.~Jovanovic, E.~K. Kogias, N.~Gailly, L.~Gasser, I.~Khoffi, M.~J.
  Fischer, and B.~Ford, ``Scalable bias-resistant distributed randomness,'' in
  \emph{2017 IEEE Symposium on Security and Privacy (SP)}, 2017, pp. 444--460.

\bibitem{diffie1976new}
W.~Diffie and M.~Hellman, ``New directions in cryptography,'' \emph{IEEE Trans.
  on Inf. Theory}, vol.~22, no.~6, pp. 644--654, 1976.

\bibitem{so2019codedprivateml}
J.~So, B.~Guler, A.~S. Avestimehr, and P.~Mohassel, ``Codedprivateml: A fast
  and privacy-preserving framework for distributed machine learning,''
  \emph{arXiv preprint arXiv:1902.00641}, 2019.

\bibitem{so2020scalable}
J.~So, B.~Guler, and A.~S. Avestimehr, ``A scalable approach for
  privacy-preserving collaborative machine learning,'' in \emph{Advances in
  Neural Information Processing Systems}, 2020.

\bibitem{cover2006}
T.~M. Cover and J.~A. Thomas, \emph{Elements of Information Theory (Wiley
  Series in Telecommunications and Signal Processing)}.\hskip 1em plus 0.5em
  minus 0.4em\relax USA: Wiley-Interscience, 2006.

\bibitem{he2020fedml}
C.~He, S.~Li, J.~So, M.~Zhang, H.~Wang, X.~Wang, P.~Vepakomma, A.~Singh,
  H.~Qiu, L.~Shen \emph{et~al.}, ``Fedml: A research library and benchmark for
  federated machine learning,'' \emph{arXiv preprint arXiv:2007.13518}, 2020.

\bibitem{dalcin2011parallel}
L.~Dalc{\'\i}n, R.~Paz, and M.~Storti, ``{MPI} for {Python},'' \emph{Journal of
  Parallel and Dist. Comp.}, vol.~65, no.~9, pp. 1108--1115, 2005.

\bibitem{so2020byzantine}
J.~So, B.~G{\"u}ler, and A.~S. Avestimehr, ``Byzantine-resilient secure
  federated learning,'' \emph{IEEE Journal on Selected Areas in
  Communications}, 2020.

\bibitem{he2020group}
C.~He, S.~Avestimehr, and M.~Annavaram, ``Group knowledge transfer:
  Collaborative training of large cnns on the edge,'' \emph{Advances in Neural
  Information Processing Systems}, 2020.

\bibitem{Wassily1963Bound}
W.~Hoeffding, ``Probability inequalities for sums of bounded random
  variables,'' \emph{Journal of the American Statistical Association}, vol.~58,
  no. 301, pp. 13--30, 1963.

\bibitem{kedlaya2011fast}
K.~S. Kedlaya and C.~Umans, ``Fast polynomial factorization and modular
  composition,'' \emph{SIAM Journal on Computing}, vol.~40, no.~6, pp.
  1767--1802, 2011.

\end{thebibliography}

\clearpage

\onecolumn

\appendix
\section{Supplementary Materials}

\subsection{State-of-the-Art of Secure Aggregation}~\label{app:benchmark}
A centralized secure aggregation protocol for federated learning has been proposed by [3], where each mobile user locally trains a model. The local models are securely aggregated through a central server, who then updates the global  model. For secure aggregation, users create pairwise keys through a key exchange protocol, such as [46], then utilize them to communicate messages securely with other users, while all communication is forwarded through the server. Privacy of individual models is provided by pairwise random masking. Specifically, each pair of users $u,v\in[N]$ first agree on a pairwise random seed $s_{u,v}$. In addition, user $u$ creates a private random seed $b_u$. 
The role of $b_u$ is to prevent the privacy breaches that may occur if  user $u$ is only delayed instead of dropped (or declared as dropped by a malicious server), in which case the pairwise masks alone are not sufficient for privacy protection. 
User $u\in[N]$ then masks its model $\mathbf{x}_u$,
\begin{equation}\label{eq:google_modelmasking}
    \mathbf{y}_u \!=\! \mathbf{x}_u + \text{PRG}(b_{u})+\!\!\sum_{v:u<v}\!\text{PRG}(s_{u,v}) -\!\!  \sum_{v:u>v}\!\!\text{PRG}(s_{v,u})
\end{equation}
where $\text{PRG}$ is a pseudo random generator, and sends it to the server. 
Finally, user $u$ secret shares $b_u$ as well as $\{s_{u,v}\}_{v\in[N]}$ with the other users, via Shamir's secret sharing\footnote{The state-of-the art protocol utilizes Shamir's secret sharing to provide the robustness guarantee and privacy guarantee. It can achieve robustness guarantee to user dropout rate of up to $p=0.5$ while providing privacy guarantee against up to $T=\frac{N}{2}$ colluding users by setting the parameter $t=\frac{N}{2}+1$ in Shamir's $t$-out-of-$N$ secret sharing protocol. This allows each user to split its own random seeds into $N$ shares such that any $t$ shares can be used to reconstruct the seeds, but any set of at most $t-1$ shares reveals no information about the seeds.} [18].
From a subset of users who survived the previous stage, the server collects either the shares of the pairwise seeds belonging to a dropped user, or the shares of the private seed belonging to a surviving user (but not both). 
Using the collected shares, the server reconstructs the private seed of each surviving user, and the pairwise seeds of each dropped user, to be removed from the aggregate of the masked models. The server then computes the aggregated model,
\begin{align} \label{eq:secure_aggregation}
\mathbf{z} &= \sum_{u\in\mathcal{U}}\big(\mathbf{y}_u - \text{PRG}(b_u)\big)   -\sum_{u\in\mathcal{D}}\Big(\sum_{v:u<v}\text{PRG}(s_{u,v}) - \sum_{v:u>v}\text{PRG}(s_{v,u})\Big)
= \sum_{u\in\mathcal{U}}\mathbf{x}_u 
\end{align} 
where $\mathcal{U}$ and $\mathcal{D}$ represent the set of surviving and dropped users, respectively.

A major bottleneck in scaling secure aggregation to a large number of users is the aggregation overhead, which is quadratic in the number of users (i.e., $O(N^2)$). 
This quadratic aggregation overhead results from the fact that the server has to  reconstruct and remove the pairwise random masks corresponding to dropped users. 
In order to recover the random masks of dropped users, the server has to execute a pseudo random generator based on the recovered seeds $s_{u,v}$, which has a quadratic computation overhead as the number of pairwise masks is quadratic in the number of users, which dominates the overall time consumed in the protocol. This quadratic aggregation overhead severely limits the network size for real-world applications [10].

\subsection{An Illustrative Example}~\label{app:illustrative_example}
We next demonstrate the execution of {\name} through an illustrative example. 
Consider the network in Figure~\ref{example_fig1} with $N=9$ users partitioned into $L=3$ groups with  $N_l=3~(l\in[3])$ users per group, and assume that user $3$ in group $2$ drops during protocol execution.

{\color{Mycolor1} In the first stage, user $i\in[3]$ in group $1$ masks its model $\mathbf{x}^{(1)}_i$ using additive masking as in~(5) and computes  $\{\widetilde{\mathbf{x}}_{i,j}^{(1)}\}_{j\in[3]}$. }
Then, the user  generates the encoded models, $\{\bar{\mathbf{x}}^{(1)}_{i,j}\}_{j\in[3]}$, by using the Lagrange polynomial from~(10).  
Finally, the user initializes
$\widetilde{\mathbf{s}}^{(1)}_i = \bar{\mathbf{s}}^{(1)}_i = \mathbf{0}$,  
and sends  $\{\widetilde{\mathbf{x}}^{(1)}_{i,j}, \bar{\mathbf{x}}^{(1)}_{i,j}, \widetilde{\mathbf{s}}^{(1)}_i, \bar{\mathbf{s}}^{(1)}_i\}$ to user $j\in[3]$ in group $2$. 
Figure~\ref{example_fig2} demonstrates this stage for one user. 

\begin{figure}[t]
\centering
\includegraphics[width=0.55\linewidth]{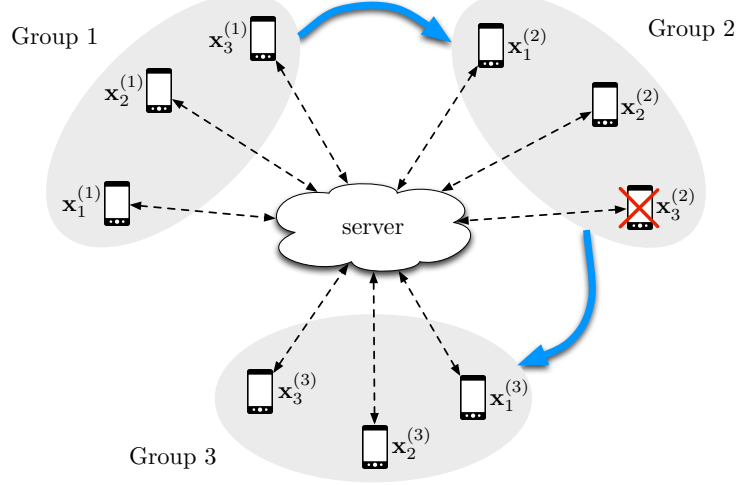}
\caption{Example with $N=9$ users and $L=3$ groups, with  $3$ users per group. User $3$ in group $2$ drops during protocol execution.}
\label{example_fig1}
\vspace{-0.2cm}\end{figure}

In the second stage, user $j\in[3]$ in group $2$ generates the masked models $\{\widetilde{\mathbf{x}}^{(2)}_{j,k}\}_{k\in[3]}$, and the coded models  $\{\bar{\mathbf{x}}^{(2)}_{j,k}\}_{k\in[3]}$,  by using~(5) and~(10), respectively. 
Next, the user aggregates the messages received from group $1$, by computing, 
$\widetilde{\mathbf{s}}^{(2)}_j = \frac{1}{3}\sum_{i\in[3]} \widetilde{\mathbf{s}}^{(1)}_i + \sum_{i\in[3]}\widetilde{\mathbf{x}}^{(1)}_{i,j}$ and 
$\bar{\mathbf{s}}^{(2)}_j = \frac{1}{3}\sum_{i\in[3]} \widetilde{\mathbf{s}}^{(1)}_i + \sum_{i\in[3]}\bar{\mathbf{x}}^{(1)}_{i,j}$. 
Figure~\ref{example_fig3} shows this aggregation phase for one user. 
Finally, user $j$ sends $\{\widetilde{\mathbf{x}}^{(2)}_{j,k}, \bar{\mathbf{x}}^{(2)}_{j,k}, \widetilde{\mathbf{s}}^{(2)}_j, \bar{\mathbf{s}}^{(2)}_j\}$ to user $k\in[3]$ in group $3$. 
However, user $3$ (in group $2$) drops out during the execution of this stage and fails to complete its part.

\begin{figure}[t]
\centering
\includegraphics[width=0.55\linewidth]{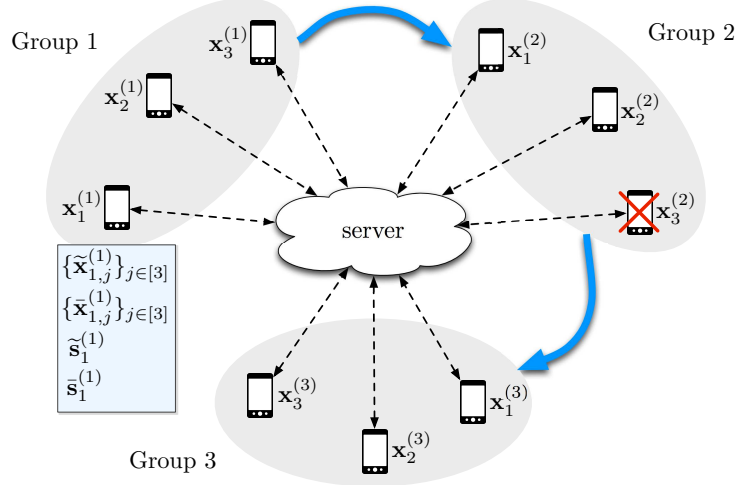}
\caption{Illustration of the computations performed by user $1$ in group $1$, who then sends  $\{\widetilde{\mathbf{x}}^{(1)}_{1,j}, \bar{\mathbf{x}}^{(1)}_{1,j}, \widetilde{\mathbf{s}}^{(1)}_1, \bar{\mathbf{s}}^{(1)}_1\}$ to user $j\in[3]$ in group $2$ (using pairwise keys through the server). }
\label{example_fig2}
\vspace{-0.4cm}\end{figure}

\begin{figure}[t]
\centering
\includegraphics[width=0.7\linewidth]{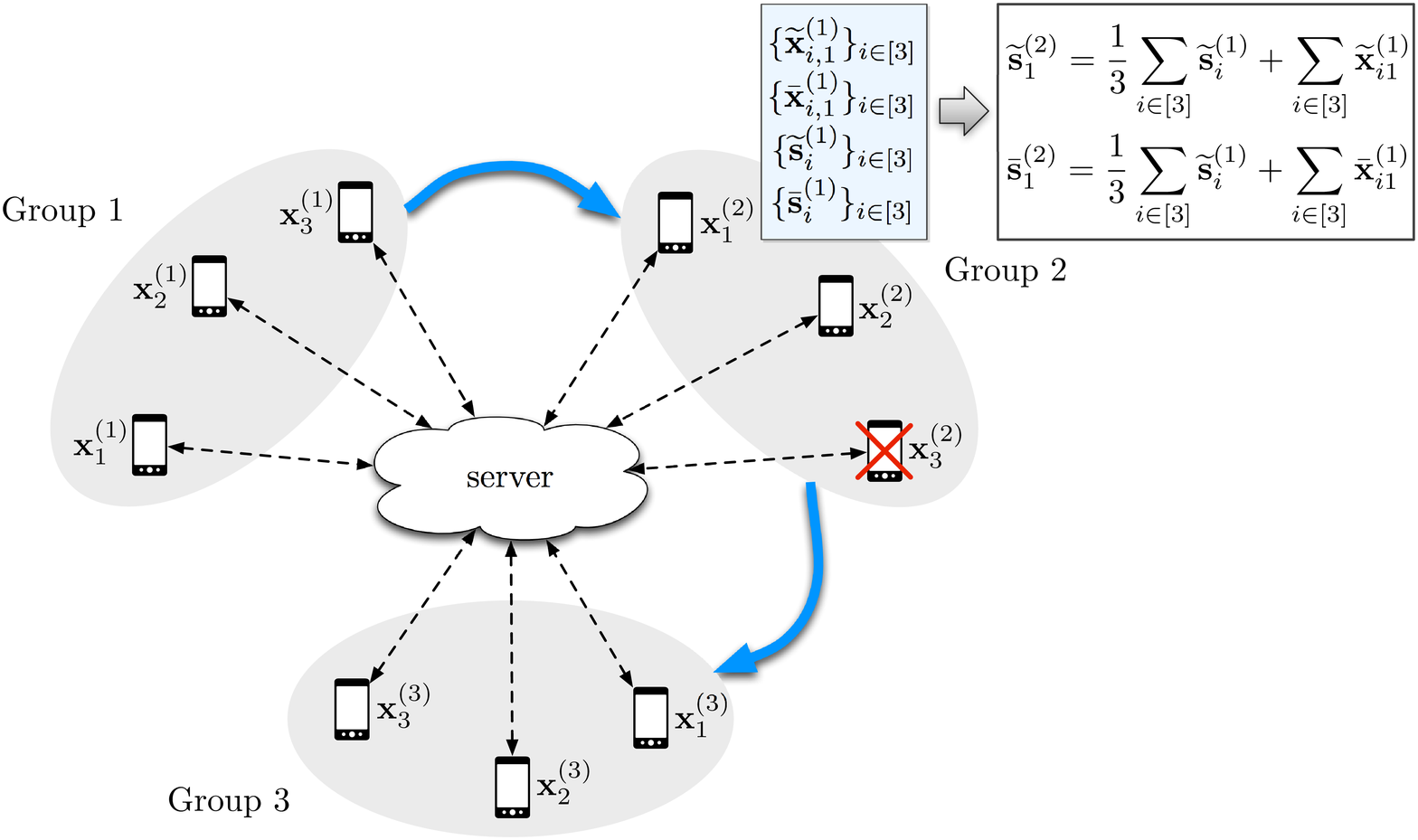}
\caption{The aggregation phase for user $1$ in group $2$. After receiving the set of messages $\big\{\widetilde{\mathbf{x}}_{i,1}^{(1)}, \bar{\mathbf{x}}_{i,1}^{(1)},  \widetilde{\mathbf{s}}_i^{(1)}, \bar{\mathbf{s}}_i^{(1)}\big\}_{i\in[3]}$ from the previous stage, the user computes the aggregated values $\widetilde{\mathbf{s}}_1^{(2)}$ and $\bar{\mathbf{s}}_1^{(2)}$ (note that this is an aggregation of the masked values).}
\label{example_fig3}
\vspace{-0.2cm}\end{figure}

In the third stage, user $k\in[3]$ in group $3$ generates the masked models $\{\widetilde{\mathbf{x}}_{k,t}^{(3)}\}_{t\in[3]}$ and the coded models $\{\bar{\mathbf{x}}_{k,t}^{(3)}\}_{t\in[3]}$.
Then, the user runs a recovery phase due to the dropped user in group $2$. 
This is facilitated by the Lagrange coding technique.  Specifically, user $k$ can decode the missing value $\widetilde{\mathbf{s}}^{(2)}_3=g^{(2)}(\alpha^{(2)}_{3})$ due to the dropped user, by using the four evaluations 
$\{\widetilde{\mathbf{s}}^{(2)}_1,\bar{\mathbf{s}}^{(2)}_1,\widetilde{\mathbf{s}}^{(2)}_2,\bar{\mathbf{s}}^{(2)}_2\} 
= \{g^{(2)}(\alpha^{(2)}_{1}),g^{(2)}(\beta^{(2)}_{1}),g^{(2)}(\alpha^{(2)}_{2}),g^{(2)}(\beta^{(2)}_{2})\}$ 
received from the remaining users in group $2$. 
Then, user $k$ aggregates the received and reconstructed values by computing      
$\widetilde{\mathbf{s}}^{(3)}_k =\frac{1}{3}\sum_{j\in[3]}  \widetilde{\mathbf{s}}^{(2)}_j + \sum_{j\in[2]}\widetilde{\mathbf{x}}^{(2)}_{j,k}$ and 
$\bar{\mathbf{s}}^{(3)}_k =\frac{1}{3}\sum_{j\in[3]}\widetilde{\mathbf{s}}^{(2)}_j + \sum_{j\in[2]}\bar{\mathbf{x}}^{(2)}_{j,k}$.

In the final stage, {\name} selects a set of surviving users to aggregate the models of group $3$. Without loss of generality, we assume these are the users in group $1$. 
Next, user $k\in[3]$ from group $3$ sends  $\{\widetilde{\mathbf{x}}^{(3)}_{k,t}, \bar{\mathbf{x}}^{(3)}_{k,t}, \widetilde{\mathbf{s}}^{(3)}_k, \bar{\mathbf{s}}^{(3)}_k\}$ to user $t\in[3]$ in the final group. 
Then, user $t$ computes the aggregation, 
$\widetilde{\mathbf{s}}^{(final)}_t = \frac{1}{3}\sum_{k\in[3]} \widetilde{\mathbf{s}}^{(3)}_k + \sum_{k\in[3]}\widetilde{\mathbf{x}}^{(3)}_{k,t}$ 
and 
$\bar{\mathbf{s}}^{(final)}_t = \frac{1}{3}\sum_{k\in[3]} \widetilde{\mathbf{s}}^{(3)}_k  + \sum_{k\in[3]}\bar{\mathbf{x}}^{(3)}_{k,t}$, and sends $\{\widetilde{\mathbf{s}}^{(final)}_t, \bar{\mathbf{s}}^{(final)}_t\}$ to the server.

Finally, the server computes the average of the summations received from the final group and removes the added randomness, 
which is equal to the aggregate of the original models of the surviving users.
\begin{align*}
    &\frac{1}{3}\sum_{t\in[3]} \widetilde{\mathbf{s}}^{(final)}_t  - \sum_{i\in[3]}\mathbf{u}^{(3)}_i - \sum_{j\in[2]}\mathbf{u}^{(2)}_j - \sum_{k\in[3]}\mathbf{u}^{(3)}_k 
     = \sum_{i\in[3]} \mathbf{x}^{(1)}_i + \sum_{j\in[2]}\mathbf{x}^{(2)}_j + \sum_{k\in[3]}\mathbf{x}^{(3)}_k. 
\end{align*} 

\subsection{Proof of Theorem~\ref{thm:TA_achieve}}~\label{app:proof_thm}
\noindent
As described in Section~\RNum{4}-A, {\name} first employs an unbiased random partitioning strategy to allocate the $N$ users into $L=\frac{N}{N_l}$ groups with a group size of $N_l$ for all $l\in[L]$. 
To prove the theorem, we choose $N_l=\frac{1}{c}\log{N}$ where $c\triangleq \min\{D(0.5||p), D(0.5||\frac{T}{N})\}$ and $D(a||b)$ is the Kullback-Leibler (KL) distance between two Bernoulli distributions with parameter $a$ and $b$ (see e.g., page $19$ of~[49]). We now prove each part of the theorem separately.

    (Robustness guarantee)
    In order to compute the aggregate of the user models, {\name} requires the partial summation in (8) to be recovered by the users in group $l\in[L]$. 
    This requires users in each group $l\in[L]$ to be able to reconstruct the term  $\widetilde{\mathbf{s}}^{(l-1)}_i$ of all the dropped users in group $l-1$. 
    This is facilitated by our encoding procedure as follows. 
    At stage $l$, each user in group $l+1$ receives $2(N_l - D_l)$ evaluations of the polynomial $g^{(l)}$, where $D_l$ is the number of users dropped in group $l$. 
    Since the degree of $g^{(l)}$ is at most $N_{l}-1$, one can reconstruct $\widetilde{\mathbf{s}}^{(l)}_i$ for all $i\in[N_l]$ using  polynomial interpolation, as long as $2(N_l - D_l) \geq N_{l}$. 
    If the condition of $2(N_l - D_l) \geq N_{l}$ is satisfied for all $l\in[L]$, {\name} can compute the aggregate of the user models.
    Therefore, the probability that {\name} provides robustness guarantee against user dropouts is given by 
    \begin{equation} \label{eq:prob_robustness}
        \mathbb{P}[\text{robustness}] = \mathbb{P} \big[ \bigcap\limits_{l\in[L]} \{D_l \leq \frac{N_l}{2}\}\big]=1-\mathbb{P} \big[ \bigcup\limits_{l\in[L]} \{D_l \geq \floor{\frac{N_l}{2}}+1\}\big].     
    \end{equation}
    
    Now, we show that this probability goes to $1$ as $N\rightarrow \infty$. $D_l$ follows a binomial distribution with parameters $N_l$ and $p$.
    When $p<0.5$, its tail probability is bounded as
    \begin{equation}\label{eq:group_fail_bound}
        \mathbb{P}[D_l \geq \floor{\frac{N_l}{2}}+1]\leq \exp{\big(-c_p N_l)},
    \end{equation} 
    where $c_p=D(\frac{\floor{\frac{N_l}{2}}+1}{N_l}||p)$~[54]. 
    When $N_l$ is sufficiently large, $c_p\approx D(0.5||p)$.
    Note that $c_p>0$ is a positive constant for any $p<0.5$, since by definition the KL distance is non-negative and equal to $0$ if and only if $p=0.5$ (Theorem $2.6.3$ of~[49]).
    Then, using a union bound, the probability of failure can be bounded by 
    \begin{align}
        \mathbb{P}[\text{failure}] 
        = \mathbb{P} \big[ \bigcup\limits_{l\in[L]} \{D_l \geq \floor{\frac{N_l}{2}}+1\}\big] \notag 
        &\leq \sum\limits_{l\in[L]} \mathbb{P} [D_l \geq \floor{\frac{N_l}{2}}+1] \notag \\
        &\leq \frac{N}{\Nl}\exp(-c_p\Nl)  \label{eq:fail_upperbound} \\
        &\triangleq B_{\text{failure}}, \label{eq:fail_bound_final1}
    \end{align}
    where~\eqref{eq:fail_upperbound} follows from~\eqref{eq:group_fail_bound}. 
    Asymptotic behavior of this upper bound is given by
    \begin{align}
        \lim_{N\rightarrow\infty}B_{\text{failure}} 
        &= \exp{\Big\{ \lim_{N\rightarrow\infty}\log{B_{\text{failure}}} \Big\}} \notag \\
        &=\exp{\Big\{\lim_{N\rightarrow\infty}\big(
        \log{N}-\log{\Nl}-c_p\Nl
        \big)\Big\}} \notag \\
        &=\exp{(-\infty)} \label{eq:minus_infty}\\
        &=0, \label{eq:fail_bound_final2}
    \end{align}
    where \eqref{eq:minus_infty} holds because $\Nl = \frac{1}{c}\log N \geq \frac{1}{c_p}\log{N}$.
    From~\eqref{eq:fail_bound_final1} and~\eqref{eq:fail_bound_final2},  $\lim_{N\rightarrow\infty}\mathbb{P}[\text{failure}]=0$. 

    
    (Privacy guarantee) 
    Let $A_l$ be an event that a collusion between users and the server can reveal the local model of any honest user in group $l-1$, and $X_l$ be a random variable corresponding to the number of colluding users in group $l\in[L]$. 
    First, note that the colluding users in groups $l'\leq l-1$ cannot learn any information about a user in group $l-1$ because communication in {\name} is directed from users in lower groups to users in upper groups. 
    Moreover, colluding users in groups $l'\geq l+1$ can learn nothing beyond the partial summation, $\mathbf{s}^{(l')}$. Hence, information breach of a local model in group $l-1$ occurs only when the number of colluding users in group $l$ is greater than or equal to half of the number of users in group $l$. In this case, colluding users can obtain a sufficient number of evaluation points to recover all of the masked models belonging to a certain user, i.e.,  $\widetilde{\mathbf{x}}^{(l-1)}_{i,j}$ in~(5) for all $j\in{N_l}$. Then, they can recover the original model  ${\mathbf{x}}^{(l-1)}_i$ by adding all of the masked models and removing the randomness ${\mathbf{u}}^{(l-1)}_i$. Therefore, $\mathbb{P}[A_l]=\mathbb{P}[X_l\geq \frac{N_l}{2}]$. 
    To calculate $\mathbb{P}[A_l]$, we again consider the random partitioning strategy described in Section~\RNum{4}-A which allocates $N$ users to $L=\frac{N}{\Nl}$ groups whose size is $\Nl$ for all groups, while $T$ out of $N$ users are colluding users.
    Then, $X_l$ follows a hypergeometric distribution with parameters $(N,T,\Nl)$ and its tail probability is bounded as 
    \begin{equation}\label{eq:wassilybound}
        \mathbb{P}[X_l \geq \frac{\Nl}{2}]\leq \exp{\big(-c_T\Nl)},
    \end{equation}
    where $c_T = D(0.5||\frac{T}{N})$~[54]. Note that $c_T>0$ is a positive constant when $T \leq (0.5 -\epsilon)N$ for any $\epsilon > 0$. Then the probability of privacy leakage of any individual model is given by
    \begin{align}
        \mathbb{P}[\text{privacy leakage}] = \mathbb{P}[ \bigcup\limits_{l\in[L]} A_{l}]  
        &\leq \sum\limits_{l\in[L]} \mathbb{P}[A_{l}] \label{eq:privacy_union_bound} \\
        &\leq \frac{N}{\Nl}\exp(-c_T\Nl)  \label{eq:privacy_upperbound} \\
        &\triangleq B_{\text{privacy}}, \notag
    \end{align}
    where \eqref{eq:privacy_union_bound} follows from a union bound and  \eqref{eq:privacy_upperbound} follows from \eqref{eq:wassilybound}.
    Note that \eqref{eq:privacy_upperbound} can be bounded similarly to  \eqref{eq:fail_upperbound} by replacing $c_p$ with $c_T$, from which we find that $\lim_{N\rightarrow \infty} B_{\text{privacy}} = 0$.
    As a result, $\lim_{N\rightarrow\infty}\mathbb{P}[\text{privacy leakage}]=0$.

    (Aggregation overhead)
    As described in Section~\RNum{3}, aggregation overhead consists of two main components, computation and communication.
    The computation overhead quantifies the  processing time for: 1) masking the models with additive secret sharing, 2) adding redundancies to the models through Lagrange coding, and 3) reconstruction of the missing information due to user dropouts.
    First, masking the model with additive secret sharing has a computation overhead of $O(\log{N})$. Second, encoding the masked models with Lagrange coding has a computation overhead of $O(\log{N}\log^2 \log{N} \log \log \log N)$, 
    because evaluating a polynomial of degree $i$ at any $j$ points has a computation overhead of $O(j \log^2 i \log \log i)$~[55], and both $i$ and $j$ are $\log{N}$ for the encoding operation. 
    Third, the decoding operation to recover the missing information due to dropped users has a computation overhead of $O(p\log{N}\log^2 \log N \log \log \log N)$, 
    because it requires evaluating a polynomial of degree $\log{N}$ at $p\log{N}$ points. 
    Within each execution stage, the computations are carried out in parallel by the users. Therefore, computations per execution stage has a computation overhead of $O\big(\log{N}\log^2 \log N \log \log \log N\big)$, which is the same as the computation overhead of a single user.     
    Since there are $L=\frac{cN}{\log{N}}$ execution stages, overall the computation overhead is $O\big(N\log^2 \log N \log \log \log N\big)$.

    Communication overhead is evaluated by measuring the total amount of communication over the entire network, to quantify the communication time in the worst-case communication architecture, which corresponds to the centralized communication architecture where all communication goes through a central server.
    At execution stage $l\in[L]$, each of the $N_l$ users in group $l$ sends a message to each of the $N_{l+1}$ users in group $l+1$, which has a communication overhead of $O(\log^2{N})$. 
    Since there are $L=\frac{cN}{\log{N}}$ stages, overall the communication overhead is $O(N\log{N})$. 
    
    As the aggregation overhead consists of the overheads incurred by communication and computation, the aggregation overhead of {\name} is $C=O(N\log{N})$. 

\subsection{Proof of Theorem~\ref{thm:converse_droprate}}~\label{app:proof_conv}
We first note that $D_l$, the number of users dropped in group $l$, follows a binomial distribution with parameters $N_l$ and $p$. The probability that the information of the users in group $l$ cannot be reconstructed is given by $\mathbb{P}\big[ D_l \geq \floor{\frac{N_l}{2}}+1 \big]$. 
From~[49], this probability can be bounded as 
\begin{equation}\label{eq:robust_bound}
    \mathbb{P}\big[ D_l \geq \floor{\frac{N_l}{2}}+1 \big] \geq \frac{1}{N_l + 1}\exp{( -c^\prime N_l)},
\end{equation}
where $c^\prime=D(0.5||p)$.

The probability that {\name} can provide robustness against user dropouts is given by
\begin{align}
    \mathbb{P}[\text{robustness}]&= \mathbb{P} \big[ \bigcap\limits_{l\in[L]} \{D_l \leq \frac{N_l}{2}\}\big] \notag \\
    &= \mathbb{P}\big[ \{D_l \leq \frac{N_l}{2}\}\big]^{L} \label{eq:indep} \\
    &= \Big( 1-\mathbb{P}\big[ D_l \geq \floor{\frac{N_l}{2}}+1 \big]\Big)^{L} \notag \\
    &\leq \Big( 1- \frac{1}{N_l + 1}\exp{( -c^\prime N_l)}\Big)^{L} \label{eq:robust_bound_applied} \\
    &= \Big( 1- \frac{1}{\log{N} + 1}\exp{( -c^\prime \log{N})}\Big)^{\frac{N}{\log{N}}} 
    \notag \\
    &= \Big( 1- \frac{N^{-c^\prime}}{\log{N} + 1}\Big)^{\frac{N}{\log{N}}} \triangleq B_{\text{robustness}}, \label{eq:robust_final}
\end{align}
where \eqref{eq:indep} holds since the events are independent over the groups and \eqref{eq:robust_bound_applied} follows from \eqref{eq:robust_bound}. 
The asymptotic behavior of the upper bound is then given by
\begin{align}
    \lim_{N\rightarrow\infty}B_{\text{robustness}} 
    &= \exp{\Big\{ \lim_{N\rightarrow\infty}\log{B} \Big\}} \notag \\
    &= \exp{\Big\{ \lim_{N\rightarrow\infty}
    \frac{ \log{\big(1- \frac{N^{-c^\prime}}{\log{N} + 1}\big)} }{\frac{\log{N}}{N}} 
    \Big\}}  
    \notag\\
    &= \exp{\Big\{ \lim_{N\rightarrow\infty}
    \frac{ \frac{c^\prime N^{-c^\prime-1}(1+\log{N})+N^{-c-1}}{(1+\log{N})(1+\log{N}-N^{-c^\prime})} }
    {\frac{1-\log{N}}{N^2}} 
    \Big\}} \label{eq:lhospital_1}\\
    &= \exp{\Big\{ \lim_{N\rightarrow\infty}
    \frac{c^\prime N^{1-c^\prime}\log{N}+N^{1-c^\prime}(1+c^\prime)}
    {(1-\log^2N)(1+\log{N}-N^{-c^\prime})}
    \Big\}} \notag \\
    &= \exp{\Big\{ \lim_{N\rightarrow\infty}
    \frac{c^\prime N^{1-c^\prime}\log{N}+N^{1-c^\prime}(1+c^\prime)}
    {-\log^3N-\log^2N+\log{N}}
    \Big\}} \label{eq:infyterm} \\
    &= \exp{\Big\{ \lim_{N\rightarrow\infty}
    \frac{c^\prime(1-c^\prime)N^{1-c^\prime}\log{N}+N^{1-c^\prime}(1+c^\prime-c^{\prime 2})}
    {-3\log^2N-2\log{N}+1}
    \Big\}} \label{eq:lhospital_2} \\
    &= \exp{\Big\{ \lim_{N\rightarrow\infty}
    \frac{c^\prime(1-c^\prime)^2 N^{1-c^\prime}\log{N}+N^{1-c^\prime}(1-c^\prime)(1+2c^\prime-c^{\prime 2})}
    {-6\log{N}-2}
    \Big\}} \label{eq:lhospital_3} \\
    &= \exp{\Big\{ \lim_{N\rightarrow\infty}
    \frac{c(1-c^\prime)^3 N^{1-c^\prime}\log{N}+N^{1-c^\prime}(1-c^\prime)^2(1+3c^\prime-c^{\prime 2})}
    {-6}
    \Big\}} \label{eq:lhospital_4} \\
    &= \exp{(-\infty)} \label{eq:goto_minus_infty} \\
    &= 0, \label{eq:bound_final}
\end{align}
where~\eqref{eq:lhospital_1},~\eqref{eq:lhospital_2},~\eqref{eq:lhospital_3}, and~\eqref{eq:lhospital_4} follow the L'Hospital's rule,~\eqref{eq:infyterm} removes the terms which do not go to infinity as $N$ goes to infinity, and~\eqref{eq:goto_minus_infty} holds since $N^{1-c^\prime}$ goes to infinity when $0<c^\prime=D(0.5||p)<1$.

From \eqref{eq:robust_final} and~\eqref{eq:bound_final}, $\lim_{N\rightarrow\infty}\mathbb{P}[\text{robustness}]=0$, which completes the proof.

\subsection{Proof of Theorem~\ref{col:modifiedTA_achieve}}~\label{app:proof_thm_gen}
The generalized version of {\name} creates two independent random partitionings of the users, where each partitioning randomly assigns $N$ users into $L$ groups with a group size of $N_l$ for all $l\in[L]$.
To prove the theorem, we choose $N_l=O(\log{N})$ as in the proof of Theorem~1 for both the first and second partitioning. 
We now prove the privacy guarantee for each individual model as well as any aggregate of a subset of models, robustness guarantee, and the aggregation overhead.

(Privacy guarantee) As the generalized protocol follows Algorithm~1, the privacy guarantee of individual models follows from the proof of the privacy guarantee in Theorem~1. We now prove the privacy guarantee of any partial aggregation, i.e., aggregate of the models from any subset of users.     
A collusion between the server and users can reveal a partial aggregation in two events: 1) for any $l,l'\in[L]$, $\mathcal{U}_{l}$ and
$\mathcal{U}^{'}_{l'}$ 
are exactly the same where $\mathcal{U}_{l}$ and $\mathcal{U}^{'}_{l'}$ denote the set of surviving users in group $l$ of the first partitioning and group $l'$ of the second partitioning, respectively, or 2) the number of colluding users in group $l'$ of the second partitioning is larger than half of the group size, and the number of such groups is large enough that colluding users can reconstruct the individual random masks $\{ \mathbf{u}^{(l)}_j\}_{j\in\mathcal{U}_{l}}$ in~(5) for some group $l$ of the first partitioning.
%
In the first event, the server can reconstruct the aggregate of the random masks, $\sum_{j\in\mathcal{U}_{l}}{\mathbf{u}}^{(l)}_j$, and then if the server colludes with any user in group $l+1$ and any user in group $l+2$, they can reveal the partial aggregation $\sum_{j\in\mathcal{U}_{l}}{\mathbf{x}}^{(l)}_j$ by subtracting $\mathbf{s}^{(l+1)}+\sum_{j\in\mathcal{U}_{l}}{\mathbf{u}}^{(l)}_j$ from $\mathbf{s}^{(l+2)}$ in~(9).
The probability that a given group from the second partitioning is the same as any group $l\in[L]$ from the first partitioning is $\frac{N}{N_l} \frac{1}{\binom{N}{N_l}}$. 
As there are $L=\frac{N}{N_l}$ groups in the second partitioning, the probability of the first event is bounded by $\frac{N^2}{N_l^2}\frac{1}{\binom{N}{N_l}}$ from a union bound, which goes to zero as $N\rightarrow\infty$ when $N_l=O(\log{N})$.
In the second event, the colluding users in group $l'$ where the number of colluding users is larger than half of the group size can reconstruct the individual random masks of all users in group $l'-1$. 
If these colluding users can reconstruct the individual random masks of all users in group $l$ for any $l\in[L]$ and collude with any user in group $l+1$ and any user in group $l+2$, they can reveal the partial aggregation $\sum_{j\in\mathcal{U}_{l}}{\mathbf{x}}^{(l)}_j$.
As the second event requires that the number of colluding users  is larger than half of the group size in multiple groups, the probability of the second event is less than the probability in~\eqref{eq:wassilybound}, which goes to zero as $N\rightarrow\infty$.
Therefore, as the upper bounds on these two events go to zero, generalized {\name} can provide privacy guarantee for any partial aggregation, i.e., aggregate of any subset of user models,  with probability approaching to $1$ as $N\rightarrow\infty$.

(Robustness guarantee) Generalized {\name} can provide the same level of robustness guarantee as Theorem~1.
This is because the probability that all groups in the first partitioning have enough numbers of surviving users for the protocol to correctly compute the aggregate $\sum_{j\in\mathcal{U}_{l}}\big({\mathbf{x}}^{(l)}_j+{\mathbf{u}}^{(l)}_j\big)$ is the same as the probability in~\eqref{eq:prob_robustness}, and the probability that all groups in the second partitioning have enough numbers of surviving users is also the same as the probability in~\eqref{eq:prob_robustness}.
Therefore, from a union bound, the upper bound on the failure probability of generalized {\name} is twice that of the bound in~\eqref{eq:fail_bound_final1}, which goes to zero as $N\rightarrow\infty$.

(Aggregation overhead) Generalized {\name} follows  Algorithm~1 which achieves an aggregation overhead of $O(N\log{N})$, and additional operations to secret share the individual random masks $\mathbf{u}^{(l)}_i$ and decode the aggregate of the random masks also have an aggregation overhead of $O(N\log{N})$. 
The aggregation overhead of the additional operations consists of four parts: 1) user $i$ in group $l$ generates the secret shares of $\mathbf{u}^{(l)}_i$ via Shamir's secret sharing with a polynomial degree $\frac{N_l}{2}$, 2) user $i$ in group $l$ sends the secret shares to the $N_l$ users in group $l+1$, 3) user $i$ in group $l$ aggregates the received secret shares, $\sum_{j\in\mathcal{U}^{'}_{l-1}}\big[{\mathbf{u}}^{(l-1)}_j\big]_i$, and sends the sum to the server, and 4) server reconstructs the secret $\sum_{j\in\mathcal{U}^{'}_{l}}{\mathbf{u}}^{(l)}_j$ for all $l\in[L]$.
The computation overhead of Shamir's secret sharing is $O(N_l^2)$~[18], and these computations are carried out in parallel by the users in one group. As there are $L=\frac{N}{N_l}$ groups and $N_l=O(\log{N})$, the computation overhead of the first part is $O(N\log{N})$.
The communication overhead of the second part is $O(N\log{N})$ as the total number of secret shares is $O(N\log{N})$.
The communication overhead of the third part is $O(N)$ as each user sends a single message to the server.
The computation overhead of the last part is $O(N\log{N})$ as the computation overhead of decoding the secret from the $N_l=O(\log{N})$ secret shares is $O(\log^2{N})$~[18], and the server must decode $L=O(\frac{N}{\log{N}})$ secrets.
Therefore, generalized {\name} also achieves an aggregation overhead of $O(N\log{N})$.

\subsection{Additional Experiments}~\label{app:add_exp}
In this section, we provide further numerical evaluations on the performance of {\name}. 
\vspace{0.3cm}
\subsubsection{Breakdown of the total running time}~\label{app:exp_breakdown}
To illustrate the impact of user dropouts, we present the breakdown of the total running time of {\name}, {\name}$+$, and the benchmark protocol in Tables~\ref{tbl:breakdown_N100},~\ref{tbl:breakdown_N100_tree_TA} and~\ref{tbl:breakdown_N100_benchmark}, respectively, for the case of $N=200$ users.   
Tables~\ref{tbl:breakdown_N100} and \ref{tbl:breakdown_N100_tree_TA} demonstrate that, for {\name} and {\name}$+$, the encoding time stays constant with respect to the user dropout rate, and decoding time is linear in the number of dropout users, which takes only a small portion of the total running time. 
Table~\ref{tbl:breakdown_N100_benchmark} shows that, for the benchmark protocol, total running time is dominated by the  reconstruction of pairwise masks (using a pseudo random generator) at the server, which has a computation overhead of $O\big((N-D) + D(N-D)\big)$ where $D$ is the number of dropout users~[3].  
This leads to an increased running time as the number of user dropouts increases. The running time of two {\name} schemes, on the other hand, is relatively stable against varying user dropout rates, as the communication time is independent from the user dropout rate and the only additional overhead comes from the decoding phase, whose overall contribution is minimal. 
Table~\ref{tbl:breakdown_N100_tree_TA} shows that the encoding time of {\name}$+$ is reduced to an $L$-th of the encoding time of original {\name} because encoding is performed in parallel across all groups.
{\name}$+$ also speeds up the decoding and communication phases by reducing the number of execution stages from $L-1$ to $\log{L}$.

It is also useful to comment on the selection of a user dropout rate of up to $p = 0.5$. From a practical perspective, our selection of $p=0.5$ is to demonstrate our results with the same parameter setting as the state-of-the-art [3]. 
From a privacy perspective, as most secure systems, e.g., Blockchain systems, are designed to tolerate at most $50\%$ adversaries, Turbo-Aggregate is also designed to achieve a  privacy guarantee against up to $T=N/2$ where $N$ is total number of users, which limits the maximum value of $p$ to $0.5$, as Turbo-Aggregate provides a trade-off between the maximum user dropout rate $p$ and the privacy guarantee $T$, as detailed in Remark 1.

\begin{table}[t!]
\small
\caption{Breakdown of the running time (ms) of {\name} with $N=200$ users. }
\label{tbl:breakdown_N100}
\begin{center}
\begin{tabular}{crrrr}
\toprule
Drop rate & Encoding & Decoding & Communication & Total \\
\midrule
$10$\% & 2070 & 2333 & 22422  & 26825 \\
$30$\% & 2047 & 2572 & 22484  & 27103 \\
$50$\% & 2051 & 3073 & 22406  & 27530 \\
\bottomrule
\end{tabular}
\end{center}
\end{table}

\begin{table}[t!]
\small
\caption{
Breakdown of the running time (ms) of {\name}$+$ with $N=200$ users. 
}
\label{tbl:breakdown_N100_tree_TA}
\begin{center}
\begin{tabular}{crrrr}
\toprule
Drop rate & Encoding & Decoding & Communication & Total \\
\midrule
$10$\% & 93 & 356 & 3353  & 3802 \\
$30$\% & 94 & 460 & 3282  & 3836 \\
$50$\% & 94 & 559 & 3355  & 4009 \\
\bottomrule
\end{tabular}
\end{center}
\end{table}

\begin{table}[t!]
\small
\caption{Breakdown of the running time (ms) of the benchmark protocol [3] with $N=200$ users. }
\label{tbl:breakdown_N100_benchmark}
\begin{center}
\begin{tabular}{ccccr}
\toprule
Drop rate& Communication of & Reconstruction  & Other & Total \\
& the models & at server &      & \\ 
\midrule
$10$\% & 8670 & 53781   & 832 & 63284 \\
$30$\% & 8470 & 101256  & 742 & 110468 \\
$50$\% & 8332 & 151183  & 800 & 160315 \\
\bottomrule
\end{tabular}
\end{center}
\end{table}

\vspace{0.3cm}
\subsubsection{Impact of bandwidth and stragglers}~\label{app:exp_params} 
In this following, we further study the impact of the bandwidth and stragglers on the performance of our protocol, by measuring the total running time with various communication bandwidth constraints. Our results are demonstrated in Figure \ref{fig:trainingtime_BW}, from which we observe that {\name} still provides substantial gain in environments with more severe bandwidth constraints.


\begin{figure}%
    \centering
    \includegraphics[width=0.5\linewidth]{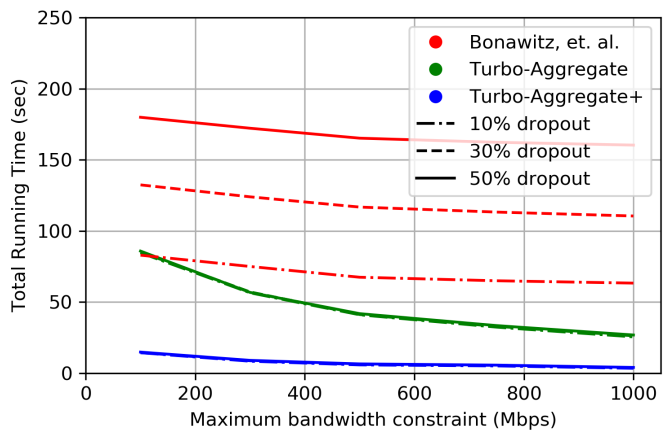} 
    \caption{Total running time of {\name} versus the benchmark protocol from [3] as the maximum bandwidth increases, for various user dropout rates. The number of users is fixed to $N=200$.}
    \label{fig:trainingtime_BW}
\end{figure}

\vspace{0.2cm}
\noindent

\vspace{0.2cm}
\noindent \textbf{Bandwidth.}
For simulating various bandwidth conditions in mobile network environments, we measure the total running time while decreasing the maximum bandwidth constraint for the communication links between the Amazon EC2 machine instances from $1 Gbps$ to $100 Mbps$.
In Figure~\ref{fig:trainingtime_BW}, we observe that the gain of {\name} and {\name}$+$ over the benchmark decrease as the maximum bandwidth constraint decreases. 
This is because for the benchmark, the major bottleneck is the running time for the  reconstruction of pairwise masks, which remains constant over various bandwidth conditions.  
On the other hand, for {\name}, the total running time is dominated by the communication time which is a reciprocal proportion to the bandwidth. 
This leads to a significantly decreased gain of {\name} over the benchmark, $1.9\times$ with the maximum bandwidth constraint of $100 Mbps$. 
However, total running time of {\name}$+$ increases moderately as the maximum bandwidth constraint decreases because communication time of {\name}$+$ is even less than the communication time of the benchmark.
{\name}$+$ still provides a speedup of $12.1\times$ over the benchmark with the maximum bandwidth constraint of $100 Mbps$.
In real-world settings, this gain is expected to increase further as the number of users increases.

\vspace{0.2cm}
\noindent \textbf{Stragglers or delayed users.}
Beyond user dropouts, in a federated learning environment, stragglers or slow users can also significantly impact the total running time. 
{\name} can effectively handle these straggling users by simply treating them as user dropouts. 
At each stage of the aggregation, if some users send their messages later than a certain threshold, users at the higher group can start to decode those messages without waiting for stragglers. 
This has negligible impact on the total running time because {\name} provides a stable total running time as the number of dropout users increases.
For the benchmark, however, stragglers can significantly delay the total running time even though it can also handle stragglers as dropout users. 
This is because the running time for the  reconstruction of pairwise masks corresponding to the dropout users, which is the dominant time consuming part of the benchmark protocol, significantly increases as the number of dropout users increases.

\vspace{0.3cm}
\subsubsection{Additional experiments with the  federated averaging scheme (FedAvg)  [1]}

\begin{figure}[t]%
    \centering
    \includegraphics[width=0.55\linewidth]{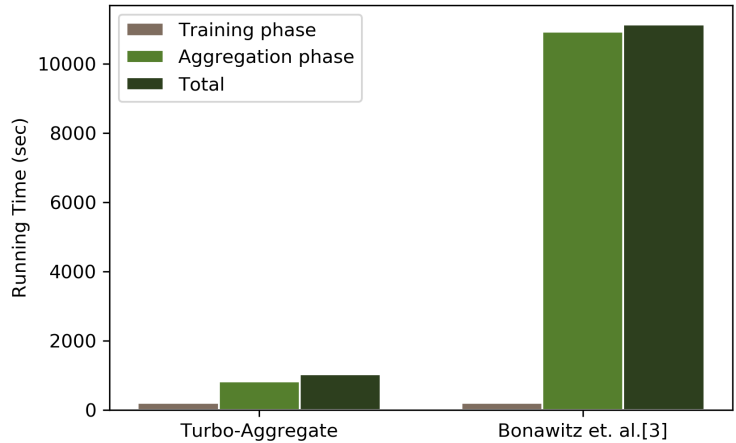}
    \vspace{-0.3cm}
    \caption{Running time of {\name} versus the benchmark protocol from [3]. Both protocols are applied to the aggregation phase in FedAvg with CIFAR-10 dataset and CNN architecture [1]. We measure the running time of training phase and aggregation phase to achieve a test accuracy of $80\%$ with $N=100$, $E=5$ (number of local epochs), and $B=50$ (mini-batch size).}
    \label{fig:realtask}
\end{figure}

In Section \RNum{6}, we have primarily focused on the aggregation phase and measured a single round of the secure aggregation phase with synthesized vectors for the locally trained models.
This is due to the fact that these vectors can be replaced with any trained model using the real world federated learning setups. 

To further investigate the performance of {\name} in real world federated learning setups, we implement the  \emph{FedAvg} scheme from [1] with a convolutional neural network (CNN) architecture on the  CIFAR-10 dataset as considered in [1],  and apply two secure aggregation protocols, {\name} and the benchmark protocol from [3], in the aggregation phase.
This architecture has $100,\!000$ parameters, and with the setting of $N=100$, $E=5$ (number of local epochs), and $B=50$ (mini-batch size), it requires $280$ rounds to achieve a test accuracy of $80\%$ [1].
Figure \ref{fig:realtask} shows the local training time, aggregation time, and total running time of Turbo-Aggregate and the benchmark protocol [3], from which we have  observed that Turbo-Aggregate provides $10.8\times$ speedup over the benchmark to achieve $80\%$ test accuracy.
We note that this gain can be much larger (almost $40\times$) when the number of users is larger ($N=200$).

\end{document}